\documentclass[reqno,11pt]{article} 
\usepackage[utf8]{inputenc}

\usepackage{txie}

\usepackage[T1]{fontenc}    
\usepackage{url}            
\usepackage{booktabs}       
\usepackage{amsfonts}       
\usepackage{amsmath,amsthm,amssymb,bbm}
\usepackage{mathtools}
\usepackage{nicefrac}       
\usepackage{microtype}      
\usepackage{enumitem}
\usepackage{mathrsfs}
\usepackage{hyperref}
\hypersetup{colorlinks=true,linkcolor=blue,citecolor=darkblue}

\usepackage[english]{babel}
\usepackage{dsfont}

\usepackage{comment}
\usepackage{caption}

\usepackage{algorithm,algorithmic}

\usepackage{color}
\usepackage{appendix}

\usepackage{wrapfig}

\usepackage{amsthm}

\makeatletter
\newtheorem*{rep@theorem}{\rep@title}
\newcommand{\newreptheorem}[2]{%
\newenvironment{rep#1}[1]{%
 \def\rep@title{#2 \ref{##1}}%
 \begin{rep@theorem}}%
 {\end{rep@theorem}}}
\makeatother

\newreptheorem{theorem}{Theorem}
\newreptheorem{corollary}{Corollary}
\newreptheorem{lemma}{Lemma}

\renewcommand{\algorithmicrequire}{\textbf{Input:}}
\renewcommand{\algorithmicensure}{\textbf{Output:}}

\newcommand{\explain}[2]{\underset{\mathclap{\overset{\uparrow}{#2}}}{#1}}

\newtheorem{theorem}{Theorem}[section]
\newtheorem{lemma}{Lemma}[section]

\newtheorem{corollary}{Corollary}

\newtheorem{example}{Example}

\newtheorem{remark}{Remark}

\def\plim{\text{plim}}

\def\E{\mathbb{E}}
\def\P{\mathbb{P}}
\def\Cov{\mathrm{Cov}}
\def\Var{\mathrm{Var}}

\def\diag{\mathrm{diag}}

\def\R{\mathbb{R}}
\def\cA{\mathcal{A}}

\def\cP{\mathcal{P}}

\def\cS{\mathcal{S}}

\title{Towards Optimal Off-Policy Evaluation for Reinforcement Learning with Marginalized Importance Sampling}

\usepackage{authblk}
\author[1]{Tengyang Xie\thanks{Part of this work was carried out while the author worked at AWS AI Labs.}}
\author[2]{Yifei Ma}
\author[3]{Yu-Xiang Wang}
\affil[1]{Department of Computer Science, University of Illinois at Urbana-Champaign}
\affil[2]{Amazon AI}
\affil[3]{Department of Computer Science, University of California, Santa Barbara}
\affil[ ]{\texttt{tx10@illinois.edu}  \quad \texttt{yifeim@amazon.com} \quad
\texttt{yuxiangw@cs.ucsb.edu}}

\date{}

\usepackage{graphicx}
\usepackage{geometry}

\usepackage{amsmath}
\numberwithin{equation}{section} 

\usepackage{amsfonts}
\usepackage{xcolor}
\definecolor{darkblue}{HTML}{000080}

\usepackage{url}

\usepackage{mathtools}

\usepackage{mathrsfs}
\mathtoolsset{showonlyrefs}

\usepackage{natbib}
\setcitestyle{authoryear} 
\setcitestyle{square}
\setcitestyle{semicolon} 

\usepackage[labelformat=simple]{subcaption}

\usepackage{algorithm}
\usepackage{algorithmic}

\usepackage[finalold]{trackchanges}

\begin{document}

\maketitle

\setcounter{footnote}{0}




\begin{abstract}
%
%
Motivated by the many real-world applications of reinforcement learning (RL) that require safe-policy iterations, we consider the problem of off-policy evaluation (OPE) --- the problem of  evaluating a new policy using the historical data obtained by different behavior policies --- under the model of nonstationary episodic Markov Decision Processes (MDP) with a long horizon and a large action space.
Existing importance sampling (IS) methods often suffer from large variance that depends exponentially on the RL horizon $H$. To solve this problem, we consider a marginalized importance sampling (MIS) estimator that recursively estimates the state marginal distribution for the target policy at every step. 
MIS achieves a mean-squared error of 
$$
\frac{1}{n} \sum\nolimits_{t=1}^H\mathbb{E}_{\mu}\left[\frac{d_t^\pi(s_t)^2}{d_t^\mu(s_t)^2} \mathrm{Var}_{\mu}\left[\frac{\pi_t(a_t|s_t)}{\mu_t(a_t|s_t)}\big( V_{t+1}^\pi(s_{t+1}) + r_t\big) \middle| s_t\right]\right]   + \tilde{O}(n^{-1.5})
$$
where 
$\mu$ and $\pi$ are the logging and target policies, $d_t^{\mu}(s_t)$ and $d_t^{\pi}(s_t)$ are the marginal distribution of the state at $t$th step, $H$ is the horizon, $n$ is the sample size and $V_{t+1}^\pi$ is the value function of the MDP under $\pi$. The result matches the Cramer-Rao lower bound in \citet{jiang2016doubly} up to a multiplicative factor of $H$. To the best of our knowledge, this is the first OPE estimation error bound with a polynomial dependence on $H$. 
Besides theory, we show empirical superiority of our method in time-varying, partially observable, and long-horizon RL environments.

\end{abstract}



\section{Introduction}

The problem of \emph{off-policy evaluation} (OPE), which predicts the performance of a policy with data only sampled by a behavior policy \citep{sutton1998reinforcement}, is crucial for using \emph{reinforcement learning} (RL) algorithms responsibly in many real-world applications. In many settings where RL algorithms have already been deployed, e.g., targeted advertising and marketing \citep{bottou2013counterfactual,tang2013automatic,chapelle2015simple,theocharous2015personalized,thomas2017predictive} or medical treatments \citep{murphy2001marginal, ernst2006clinical, raghu2017continuous}, online policy evaluation is usually expensive, risky, or even unethical. Also, using a bad policy in these applications is dangerous and could lead to severe consequences. Solving OPE is often the starting point in many RL applications.

To tackle the problem of OPE, the idea of importance sampling (IS) corrects the mismatch in the distributions under the behavior policy and target policy. It also provides typically unbiased or strongly consistent estimators \citep{precup2000eligibility}. 
IS-based off-policy evaluation methods have also seen lots of interest recently especially for short-horizon problems, including contextual bandits \citep{murphy2001marginal,hirano2003efficient,dudik2011doubly,wang2017optimal}.
However, the variance of IS-based approaches \citep{precup2000eligibility,thomas2015high,jiang2016doubly,thomas2016data,guo2017using,farajtabar2018more} tends to be too high to provide informative results, for long-horizon problems \citep{mandel2014offline}, since the variance of the product of importance weights may grow exponentially as the horizon goes long. There are also model-based approaches for solving OPE problems \citep{liu2018representation, gottesman2019combining}, where the value of the target policy is estimated directly using the approximated MDP.

Given this high-variance issue, it is necessary to find an IS-based approach without relying heavily on the cumulative product of importance weights from the whole trajectories.
While the benefit of cumulative products is to allow unbiased estimation even without any state observability assumptions,
reweighing the entire trajectories may not be necessary if some intermediate states are directly observable.
For the latter, based on Markov independence assumptions, we can aggregate all trajectories that share the same state transition patterns to directly estimate the state distribution shifts after the change of policies from the behavioral to the target.
We call this approach marginalized importance sampling (MIS),
because it computes the \emph{marginal} state distribution shifts at every single step, instead of the product of policy weights.

Related work \citep{liu2018breaking} tackles the high variance issue due to the cumulative product of importance weights. They apply importance sampling on the average visitation distribution of state-action pairs,
based on an estimation of the mixed state distribution.
%
\citet{hallak2017consistent} and \citet{gelada2019off} also leverage the same fact in time-invariant MDPs, where they use the stationary ratio of state-action pairs to replace the trajectory weights.
%
%
However, these methods may not directly work in finite-horizon MDPs, where the state distributions may not mix.

In contrast to the prior work, the first goal of our paper is to study the sample complexity and optimality of the marginalized approach. Specifically, we provide the first finite sample error bound on the mean-square error for our MIS off-policy evaluation estimator under the episodic tabular MDP setting (with potentially continuous action space). Our MSE bound is the exact calculation up to low order terms. Comparing to the Cramer-Rao lower bound established in \citep[Theorem 3]{jiang2016doubly} for DAG-MDP, our bound is larger by at most a factor of $H$ and we have good reasons to believe that this additional factor is required for any OPE estimators in this setting.


In addition to the theoretical results, we empirically evaluate our estimator against a number of strong baselines from prior work in a number of time-invariant/time-varying, fully observable/partially observable, and long-horizon environments.
Our approach can also be used in most of OPE estimators that leverage IS-based estimators,
such as doubly robust \citep{jiang2016doubly}, MAGIC \citep{thomas2016data}, MRDR \citep{farajtabar2018more} under mild assumptions (Markov assumption).

Here is a road map for the rest of the paper. 
Section~\ref{sec:bg} provides the preliminaries of the problem of off-policy evaluation. 
%
In Section~\ref{sec:marope}, we offer the design of our marginalized estimator, and we study its information-theoretical optimality in Section~\ref{sec:theo_opt}.
%
%
We present the empirical results in a number of RL tasks in Section~\ref{sec:exp}.
At last, Section~\ref{sec:con} concludes the paper.


\section{Problem formulation}
\label{sec:bg}

\paragraph{Symbols and notations.}
We consider the problem of off-policy evaluation for a finite horizon, nonstationary, episodic MDP, which is a tuple defined by $M = (\mathcal{S}, \mathcal{A}, T, r, H)$, where $\mathcal S$ is the state space, $\mathcal A$ is the action space, $T_t: \mathcal S \times \mathcal A \times \mathcal S \to [0,1]$ is the \emph{transition function} with $T_t(s' | s,a)$ defined by probability of achieving state $s'$ after taking action $a$ in state $s$ at time $t$, and $r_t: \mathcal S \times \mathcal A \times \cS \to \R$ is the expected reward function with $r_t(s,a,s')$ defined by the mean of immediate received reward after taking action $a$ in state $s$ and transitioning into $s'$, and $H$ denotes the finite horizon.
We use $\P[E]$ to denote the probability of an event $E$ and $p(x)$ the p.m.f. (or pdf) of the random variable $X$ taking value $x$. $\E[\cdot]$ and $\E[\cdot|E]$ denotes the expectation and conditional expectation given $E$, respectively.

Let $\mu,\pi: \cS \rightarrow \P_\cA$ be policies which output a distribution of actions given an observed state. We call $\mu$ the behavioral policy and $\pi$ the target policy. For notation convenience we denote $\mu(a_t|s_t)$ and $\pi(a_t|s_t)$ the p.m.f of actions given state at time $t$. The expectation operators in this paper will either be indexed with $\pi$ or $\mu$, which denotes that all random variables coming from roll-outs from the specified policy.  Moreover, we denote $d_t^{\mu}(s_t)$ and $d_t^{\pi}(s_t)$ the induced state distribution at time $t$. When $t=1$, the initial distributions are identical $d_1^\mu=d_1^\pi=d_1$. For $t>1$, $d_t^{\mu}(s_t)$ and $d_t^{\pi}(s_t)$ are functions of not just the policies themselves but also the unknown underlying transition dynamics, i.e., for $\pi$ (and similarly $\mu$), recursively define
\begin{equation}
\begin{aligned}
	&d_t^{\pi}(s_t)=\sum_{s_{t-1}}  P_t^\pi(s_t|s_{t-1}) d_{t-1}^\pi(s_{t-1}),
	\\
	&\qquad \mbox{ where }
	P_t^\pi(s_t|s_{t-1}) = \sum_{a_{t-1}} T_t(s_t|s_{t-1}, a_{t-1})\pi(a_{t-1} | s_{t-1}).
	\label{eq:state-distribution}
\end{aligned}
\end{equation}
We denote $P_{i,j}^\pi\in \R^{S\times S} ~\forall j<i$ as the state-transition probability from step $j$ to step $i$ under a sequence of actions taken by $\pi$.  Note that $P_{t+1,t}^\pi(s'|s) =  \sum_{a}P_{t+1,t}(s'|s,a)\pi_t(a|s)= T_{t+1}(s'|s,\pi_t(s))$.

Behavior policy $\mu$ is used to collect data in the form of $(s^{(i)}_t,a^{(i)}_t,r^{(i)}_t)\in \cS\times \cA\times \R$ for time index $t=1,\dotsc,H$ and episode index $i=1,...,n$. Target policy $\pi$ is what we are interested to evaluate. 
Also, let $\mathcal D$ to denote the historical data, which contains $n$ episode trajectories in total.
We also define $\mathcal D_h=\{(s_t^{(i)},a_t^{(i)},r_t^{(i)}):i\in[n],t\leq h\}$ to be roll-in realization of $n$ trajectories up to step $h$.

Throughout the paper, probability distributions are often used in their vector or matrix form. For instance, $d_t^\pi$ without an input is interpreted as a vector in a $S$-dimensional probability simplex and $P_{i,j}^\pi$ is then a stochastic transition matrix. This allows us to write \eqref{eq:state-distribution} concisely as $d_{t+1}^\pi = P_{t+1,t}^\pi d_t^\pi$.

Also note that while $s_t,a_t,r_t$ are usually used to denote fixed elements in set $\cS,\cA$ and $\R$, in some cases we also overload them to denote generic random variables $s_t^{(1)}, a_t^{(1)}, r_t^{(1)}$. For example, $\E_\pi[ r_t] = \E_\pi[ r_t^{(1)}] =\sum_{s_t,a_t,s_{t+1}} d^\pi(s_t,a_t,s_{t+1})r_t(s_t,a_t,s_{t+1})$ and $\Var_\pi[r_t(s_t,a_t,s_{t+1})] =  \Var_\pi[r_t(s_t^{(1)},a_t^{(1)},s_{t+1}^{(1)})]$. The distinctions will be clear in each context.


\paragraph{Problem setup.}
The problem of off-policy evaluation is about finding an estimator $\widehat{v}^\pi:  (\cS\times \cA\times \R)^{H\times n} \rightarrow \R$ that makes use of the data collected by running $\mu$ to estimate
\begin{align}
& v^{\pi} = 
\sum_{t=1}^{H} \sum_{s_t} d_t^\pi(s_t) \sum_{a_t} \pi(a_t|s_t) \sum_{s_{t+1}}T_t(s_{t+1}|s_t,a_t)r_t(s_t, a_t,s_{t+1}),
\label{eq:objective}
\end{align}
where we assume knowledge about $\mu(a|s)$ and $\pi(a|s)$ for all $(s,a)\in \cS\times \cA$,
but \emph{do not observe} $r_t(s_t, a_t,s_{t+1})$ for any actions other than a noisy version of it the evaluated actions. Nor do we observe the state distributions $d_t^\pi(s_t) \forall t>1$ implied by the change of policies.
Nonetheless, our goal is to find an estimator to minimize the mean-square error (MSE):
$
\textrm{MSE}(\pi,\mu, M)
= \mathbb{E}_\mu[(\hat v^\pi - v^\pi)^2],
%
$
using the observed data and the known action probabilities.
%
%
%
%
Different from previous studies, we focus on the case where $S$ is sufficiently small but $S^2A$ is too large for a reasonable sample size. In other words, this is a setting where we do not have enough data points to estimate the state-action-state transition dynamics, but we do observe the states and can estimate the 
distribution of the states after the change of policies, which is our main strategy.
\paragraph{Assumptions:} We list the technical assumptions we need and provide necessary justification.
\begin{itemize}
    \item[A1.] $\exists R_{\max},\sigma<+\infty$ such that  
    $0\leq \E[r_t | s_t,a_t,s_{t+1}]\leq R_{\max}, \Var[r_t | s_t,a_t,s_{t+1}] \leq \sigma^2$ for all $t,s_t,a_t$.
    \item[A2.]  Behavior policy $\mu$ obeys that $d_m\coloneqq\min_{t,s_t} d_t^\mu(s_t)>0 \quad\forall t,s_t \text{ such that } d_t^\pi(s_t) > 0$.
    \item[A3.] Bounded weights:  $\tau_s\coloneqq\max_{t,s_t}\frac{d_t^\pi(s_t)}{d_t^\mu(s_t)} < +\infty$ and $\tau_a\coloneqq\max_{t,s_t,a_t}\frac{\pi(a_t|s_t)}{\mu(a_t|s_t)} <+\infty$.
\end{itemize}
Assumption A1 is assumed without loss of generality. The $\sigma$ bound is required even for on-policy evaluation and the assumption on the non-negativity and $R_{\max} $ can always be obtained by shifting and rescaling the problem. 
Assumption A2 is necessary for any consistent off-policy evaluation estimator. Assumption A3 is also necessary for discrete state and actions, as otherwise the second moments of the importance weight would be unbounded. For continuous actions, $\tau_a<+\infty$ is stronger than we need and should be considered a simplifying assumption for the clarity of our presentation. Finally, we comment that the dependence in the parameter $d_m,\tau_s,\tau_a$ do not occur in the leading $O(1/n)$ term of our MSE bound, but only in simplified results after relaxation.

%

\section{Marginalized Importance Sampling Estimators for OPE}
\label{sec:marope}

In this section, we present the design of marginalized IS estimators for OPE. For small action spaces, we may directly build models by the estimated transition function $T_t(s_t|s_{t-1}, a_{t-1})$ and the reward function $r_t(s_t,a_t,s_{t+1})$ from empirical data.
However, the models may be inaccurate in large action spaces, where not all actions are frequently visited.
Function approximation in the models may cause additional biases from covariate shifts due to the change of policies. Standard importance sampling estimators (including the doubly robust versions)\citep{dudik2011doubly,jiang2016doubly} avoid the need to estimate the model's dynamics but rather directly approximating the expected reward:
\begin{equation}\label{eq:IS_value}
    \widehat v_{\mathrm{IS}}^\pi = \frac{1}{n}
    \sum_{i=1}^n\sum_{h=1}^H 
    \left[
    \prod_{t=1}^h \frac{\pi(a_t^{(i)}|s_t^{(i)})}{\mu(a_t^{(i)}|s_{t}^{(i)})}
    \right]
    r_h^{(i)}.
\end{equation}
To adjust for the differences in the policy, importance weights are used and it can be shown that this is an unbiased estimator of $v^\pi$ (See more detailed discussion of IS and the doubly robust version in Appendix~\ref{app:ope_framework}). 
The main issue of this approach, when applying to the episodic MDP with large action space is that the variance of the importance weights grows exponentially in $H$ \citep{liu2018breaking}, which makes the sample complexity exponentially worse than the model-based approaches, when they are applicable.
We address this problem by proposing an alternative way of estimating the importance weights which achieves the same sample complexity as the model-based approaches while allowing us to achieve the same flexibility and interpretability as the IS estimator that does not explicitly require estimating the state-action dynamics $T_t$. 
We propose the Marginalized Importance Sampling (MIS) estimator:
\begin{equation}\label{eq:MIS_value}
    \widehat{v}_{\mathrm{MIS}}^\pi = \frac{1}{n} \sum_{i=1}^n
    \sum_{t=1}^H
    \frac{\widehat{d}_t^\pi(s_t^{(i)})}{\widehat{d}_t^\mu(s_t^{(i)})} \widehat{r}_t^\pi(s_t^{(i)}).
\end{equation}
Clearly, if $\widehat{d}^\pi\rightarrow d_t^\pi$, $\widehat{d}^\mu\rightarrow d_t^\mu$, $\widehat{r}_t^\pi\rightarrow \E_\pi[R_t(s_t,a_t)|s_t]$, then $\widehat{v}_{\mathrm{MIS}}^\pi \rightarrow v^\pi$.

It turns out that if we take $\widehat{d}_t^\mu(s_t):= \frac{1}{n}\sum_{i}\mathbf{1}(s_t^{(i)}=s_t)$ --- the empirical mean --- and define $\widehat{d}_t^\pi(s_t)/\widehat{d}_t^\mu(s_t)=0$ whenever $n_{s_t}=0$, then \eqref{eq:MIS_value} is equivalent to $\sum_{t=1}^H\sum_{s_t} \widehat{d}_t^\pi(s_t) \widehat{r}^\pi(s_t)$ -- the direct plug-in estimator of \eqref{eq:objective}.
It remains to specify $\widehat{d}_t^\pi(s_t)$ and $\widehat{r}^\pi(s_t)$. $\widehat{d}_t^\pi(s_t)$ is estimated recursively using
\begin{align}
&\widehat{d}_t^\pi  = \widehat{P}^{\pi}_t \widehat{d}_{t-1}^\pi,
\text{ where }
	\widehat{P}^{\pi}_t(s_{t} | s_{t-1}) =  \frac{1}{n_{s_{t-1}}} \sum_{i=1}^{n}   \frac{\pi( a_{t-1}^{(i)}| s_{t-1})}{\mu( a_{t-1}^{(i)}| s_{t-1})}  \mathbf{1}((s_{t-1}^{(i)},s_t^{(i)}) = (s_{t-1},s_t) );\nonumber
	\\
	&\mbox{ and }
	\widehat{r}_t^{\pi}(s_t)  =  \frac{1}{n_{s_t}}\sum_{i=1}^n \frac{\pi(a_t^{(i)}|s_t)}{\mu(a_t^{(i)}|s_t)} r_t^{(i)}  \mathbf{1}(s_t^{(i)} = s_t),
	\label{eq:mis-reward}
\end{align}
where $n_{s_\tau}$ is the empirical visitation frequency to state $s_\tau$ at time $\tau$. Note that our estimator of $r_t^\pi(s_t)$ is the standard IS estimators we use in bandits \citep{li2015toward}, which are shown to be optimal (up to a universal constant) when $A$ is large \citep{wang2017optimal}. 

The advantage of MIS over the naive IS estimator is that the variance of the importance weight need not depend exponentially in $H$. A major theoretical contribution of this paper is to formalize this argument by characterizing the dependence on $\pi,\mu$ as well as parameters of the MDP $M$. 
Note that MIS estimator does not dominate the IS estimator. In the more general setting when the state is given by the entire history of observations, \citet{jiang2016doubly} establishes that no estimators can achieve polynomial dependence in $H$. We give a concrete example later (Example~\ref{ex:IS-exponential}) about how IS estimator suffers from the ``curse of horizon'' \citep{liu2018breaking}. MIS estimator can be thought of as one that exploits the state-observability while retaining properties of the IS estimators to tackle the problem of large action space. As we illustrate in the experiments, MIS estimator can be modified to naturally handle \emph{partially observed} states, e.g., when $s$ is only observed every other step.

Finally, when available, model-based approaches can be combined into importance-weighted methods
\citep{jiang2016doubly,thomas2016data}.
We defer discussions about these extensions in Appendix~\ref{app:ope_framework} to stay focused on the scenarios where model-based approaches are not applicable.


\section{Theoretical Analysis of the MIS Estimator}
\label{sec:theo_opt}

Motivated by the challenge of curse of horizon with naive IS estimators,
similar to \citep{liu2018breaking},
we show that the sample complexity of our MIS estimator reduces to $O(H^3)$.
To the best of our knowledge, this is first sample complexity guarantee under this setting, which also matches the Cramer-Rao lower bound for DAG-MDP \citep{jiang2016doubly} as $n\rightarrow \infty$ up to a factor of $H$.

\begin{example}[Curse of horizon]
\label{ex:IS-exponential}
Assume a MDP with i.i.d.~state transition models over time and assume that $\frac{\pi_t}{\mu_t}$ is bounded from both sides for all $t$.
Suppose the reward is a constant $1$ only shown at the last step, such that naive IS becomes
$
    \widehat v_{\mathrm{IS}}^\pi = \frac{1}{n}
    \sum_{i=1}^n
    \left[
    \prod_{t=1}^H \frac{\pi(a_t^{(i)}|s_t^{(i)})}{\mu(a_t^{(i)}|s_{t}^{(i)})}
    \right].
$
For every trajectory,
$\prod_{t=1}^H\frac{\pi_t}{\mu_t}
=\exp\left[\sum_{t=1}^{H}\log\frac{\pi_t}{\mu_t}\right]$;
let
$E_{\log}=\mathbb{E}[\log\frac{\pi_t}{\mu_t}]$
and
$V_{\log}=\Var[\log\frac{\pi_t}{\mu_t}]$.
By Central Limit Theorem, $\sum_{t=1}^{H}\log\frac{\pi_t}{\mu_t}$ asymptotically follows a normal distribution with parameters $\bigl(-HE_{\log}, HV_{\log}\bigr)$.
In other words, 
$\prod_{t=1}^{H}\frac{\pi_t}{\mu_t}$
asymptotically follows
$
    \mathrm{LogNormal}\bigl(-HE_{\log}, HV_{\log}\bigr),
$
whose variance is exponential in horizon:
$\bigl(\exp\left(HV_{\log}\right)-1\bigr)$.
On the other hand, MIS estimates the state distributions recursively, yielding variance that is polynomial in horizon and small OPE errors.
\end{example}

We now formalize the sample complexity bound in Theorem~\ref{thm:main}.

\begin{theorem}\label{thm:main}	
	Let the value function under $\pi$ be defined as follows:
\begin{align*}
&V_h^\pi(s_h) :=   \E_\pi\left[ \sum_{t=h}^{H}r_t(s_t^{(1)},a_t^{(1)},s_{t+1}^{(1)}) \middle| s_{h}^{(1)}=s_h\right] \in [0,V_{\max}],\;  \forall h\in\{1,2,...,H\}.
\end{align*}
For the simplicity of the statement, define boundary conditions:
$r_0(s_0) \equiv 0$, $\sigma_0(s_0,a_0)\equiv 0$,$\frac{d_0^\pi(s_0)}{d_0^\mu(s_0)}\equiv 1$, $\frac{\pi(a_0|s_0)}{\mu(a_0|s_0)} \equiv 1$ and $V_{H+1}^\pi \equiv 0$.
Moreover, let 	$\tau_a := \max_{t,s_t,a_t}\frac{\pi(a_t|s_t)}{\mu(a_t|s_t)}$ and $\tau_s := \max_{t,s_t} \frac{d_t^\pi(s_t)}{d_t^\mu(s_t)}$. 
If the number of episodes $n$ obeys that
$$n > \max\left\{ \frac{16\log n}{\min_{t,s_t}d_t^\mu(s_t)},\frac{4t \tau_a \tau_s}{\min_{t,s_t}\max\{d_{t}^\pi(s_{t}),d_{t}^\mu(s_{t})\} }  \right\}$$ 
for all $t=2,...,H$, then the our estimator $\widehat{v}_{\mathrm{MIS}}^\pi$ with an additional clipping step obeys that
\begin{align*}
\E[ (\cP\widehat{v}_{\mathrm{MIS}}^\pi -  v^\pi)^2] \leq&  \frac{1}{n}\sum_{h=0}^H \sum_{s_h}  \frac{ d_{h}^\pi(s_h)^2}{d_{h}^\mu(s_h)} \Var_{\mu}\left[\frac{\pi(a_h^{(1)}|s_h)}{\mu(a_h^{(1)}|s_h)} (V_{h+1}^\pi(s_{h+1}^{(1)}) +  r_h^{(1)})\middle| s_{h}^{(1)}=s_h\right] \\
&\cdot \left(1+\sqrt{\frac{16\log n}{n\min_{t,s_t}d_t^\mu(s_t)}}\right) + \frac{19\tau_a^2\tau_s^2 S H^2(\sigma^2 + R_{\max}^2 + V_{\max}^2)}{n^2}. 
\end{align*}
\end{theorem}
\begin{corollary}\label{cor:simple_bound}
	In the familiar setting when $V_{\max} = H R_{\max} $, then the same conditions in Theorem~\ref{thm:main} implies that:
	$$\E[ (\cP \widehat{v}_{\mathrm{MIS}}^\pi -  v^\pi)^2] \leq \frac{4}{n}\tau_a\tau_s(H\sigma^2 + H^3R_{\max}^2).$$
\end{corollary}


We make a few remarks about the results in Theorem~\ref{thm:main}.

\noindent\textbf{Dependence on $S,A$ and the weights.} The leading term in the variance bound very precisely calculates the MSE of a clipped version of our estimator $\widehat{v}_{\mathrm{MIS}}$\footnote{The clipping step to $[0, HR_{\max}]$ or $[0,V_{\max}]$ should not be alarming. It is required only for technical reasons, and the clipped estimator is a valid estimator to begin with. Since the true policy value must be within the range, the clipping step is only going to improve the MSE. }  modulo a $(1+O(n^{-1/2}))$ multiplicative factor and an $O(1/n^2)$ additive factor. Specifically, our bound does not explicitly depend on $S$ and $A$ but instead on how similar $\pi$ and $\mu$ are. This allows the method to handle the case when the action space is continuous. The dependence on $\tau_a,\tau_s$ only appear in the low-order terms, while the leading term depends only on the second moments of the importance weights.

\noindent\textbf{Dependence on $H$.} In general, our sample complexity upper bound is proportional to $H^3$, as Corollary~\ref{cor:simple_bound} indicates. Our bound reveals that in several cases it is possible to achieve a smaller exponent on $H$ for specific triplets of $(M,\pi,\mu)$. For instance, when $\pi \approx \mu$, such that $\tau_a,\tau_s = 1 + O(1/H)$, the variance bound gives $O((V_{\max }^2+ H\sigma^2)/n)$ or $O((H^2R_{\max }^2 + H\sigma^2)/n)$, which matches the MSE bound (up to a constant) of the simple-averaging estimator that knows $\pi=\mu$ a-priori. (See Remark~\ref{rmk:pi_eq_mu_case} in the Appendix for more details). If $V_{\max}$ is a constant that does not depend on $H$ (this is often the case in games when there is a fixed reward at the end), then the sample complexity is only $O(H)$.


\noindent\textbf{Optimality.}  
Comparing to the Cramer-Rao lower bound of the Theorem 3 in \citep{jiang2016doubly}, which we paraphrase below
\begin{equation}\label{eq:cr_lowerbound_statement}
\frac{1}{n}\sum_{h=1}^H\sum_{s_h} \frac{d_h^\pi(s_h)^2}{d_h^\mu(s_h)}\sum_{a_h}\frac{\pi_h(a_{h}|s_h)^2}{\mu_h(a_{h}|s_h)}\Var\left[V_{h+1}^\pi(s_{h+1}^{(1)}) + r_h^{(1)}\middle|  s_{h}^{(1)} = s_h, a_{h}^{(1)} = a_h\right],
\end{equation}
the MSE of our estimator is asymptotically bigger by an additive factor of \begin{equation}\label{eq:gap_from_cr_lowerbound}
\frac{1}{n}\sum_{h=1}^H\sum_{s_h} \frac{d_h^\pi(s_h)^2}{d_h^\mu(s_h)}\Var_\mu\left[\frac{\pi_h(a_{h}^{(1)}|s_h)}{\mu_h(a_{h}^{(1)}|s_h)}
Q_h^\pi(s_{h}, a_{h}^{(1)})\right],
\end{equation}
where $Q_h^\pi(s_h,a_h) := \E \big[(V_{h+1}^\pi(s_{h+1}^{(1)}) +  r_h^{(1)})\big| s_{h}^{(1)}=s_h,a_{h}^{(1)}=a_h\big]$ is the standard $Q$-function the MDP. The gap is significant as the CR lower bound \eqref{eq:cr_lowerbound_statement} itself only has a worst-case bound of $H^2\tau_s\tau_a/n$ \footnote{This is somewhat surprising as each of the $H$ summands in the expression can be as large as $H^2$.}, while \eqref{eq:gap_from_cr_lowerbound} is proportional to $H^3\tau_s\tau_a/n$. This implies that our estimator is optimal up to a factor of $H$.  See Remark~\ref{rmk:cr_lowerbound} for more details in the appendix.

It is an intriguing open question whether this additional factor of $H$ can be removed. Our conjecture is that the answer is negative and what we established in Theorem~\ref{thm:main} matches the \emph{correct} information-theoretic limit for any methods in the cases when the action space $\cA$ is continuous (or significantly larger than $n$).
This conjecture is consistent with an existing lower bound in the simpler contextual bandits setting, where \citet{wang2017optimal} established that a variance of expectation term analogous to the one above cannot be removed, and no estimators can asymptotically attain the CR lower bound for all problems in the large state/action space setting.

\subsection{Proof Sketch}
In this section, we briefly describe the main technical components in the proof of Theorem~\ref{thm:main}.  More detailed arguments are deferred to the full proof in Appendix~\ref{sec:full_theo}. 

Recall that \eqref{eq:MIS_value} is equivalent to $\sum_{t=1}^H\sum_{s_t} \widehat{d}_t^\pi(s_t) \widehat{r}^\pi(s_t)$, where 
$
\widehat{r}^\pi(s_t)
$
is estimated with importance sampling and $\widehat{d}_t^\pi(s_t)$ is recursively estimated using $\widehat{d}_{t-1}^\pi(s_{t-1})$ and the importance sampling estimator of the transition matrix $P^\pi_{t}(s_t|s_{t-1})$ under $\pi$.  While the MIS estimator is easy to state, it is not straightforward to analyze. 
We highlight three challenges below.
\begin{enumerate}
	\item[1.] \textit{Dependent data and complex estimator:} While the episodes are independent, the data within each episode are not. Each time step of our MIS estimator uses the data from all episodes up to that time step. 
	\item[2.] \textit{An annoying bias:} There is a non-zero probability that some states $s_t$ at time $t$ is not visited at all in all $n$ episodes. This creates a bias in the estimator of $\hat{d}^\pi_h$ for all time $h>t$.  While the probability of this happening is extremely small, conditioning on the high probability event breaks the natural conditional independences, which makes it hard to analyze.
		\item[3.] \textit{Error propagation:} The recursive estimator $\hat{d}_t^\pi$ is affected by all estimation errors in earlier time steps. Naive calculation of the error with a constant slack in each step can lead to a ``snowball'' effect that causes an exponential blow-up. 
\end{enumerate}
All these issues require delicate handling because otherwise the MSE calculation will not be tight. Our solutions are as follows.

\noindent\textbf{Defining the appropriate filtration.} The first observation is that we need to have a convenient representation of the data. Instead of considering the $n$ episodes as independent trajectories, it is more useful to think of them all together as a Markov chain of multi-dimensional observations of $n$ \emph{state, action, reward} triplets. Specifically, we define the ``cumulative'' data up to time $t$ by
$\text{Data}_t := \left\{s_{1:{t}}^{(i)}, a_{1:{t-1}}^{(i)},r_{1:{t-1}}^{(i)} \right\}_{i=1}^n$.
Also, we observe that the state of the Markov chain at time $t$ can be summarized by $n_{s_t}$ --- the number of times state $s_t$ is visited.

\noindent\textbf{Fictitious estimator technique.} 
We address the bias issue  by defining a fictitious estimator $\widetilde{v}^\pi$. The fictitious estimator is constructed by, instead of $\hat{d}^\pi_t$ and $\hat{r}^\pi_t$, the fictitious version of these estimators $\tilde{d}^\pi_t$ and $\tilde{r}^\pi_t$, where $\tilde{d}^\pi_t$  is constructed recursively using
$$
\tilde{d}^\pi_t(s_t) =  \sum_{s_{t-1}} \tilde{P}^\pi(s_{t}|s_{t-1})\tilde{d}^\pi_{t-1}(s_{t-1}).
$$
The key difference is that whenever $n_{s_t}< \E_\mu n_{s_t}(1-\delta)$ for some $0<\delta<1$, we assign
$ \tilde{P}^\pi(s_{t+1}|s_{t}) =  P^\pi(s_{t+1}|s_{t})$ and $\tilde{r}^\pi(s_t) = \E_\pi[r_t | s_t]$ --- the true values of interest. This ensures that the fictitious estimator is always unbiased (see Lemma~\ref{lem:unbiasedness_fictitious}). Note that this fictitious estimator cannot be implemented in practice. It is used as a purely theoretical construct that simplifies the analysis of the (biased) MIS estimator. In Lemma~\ref{lem:fictitious_approximation}, we show that the $\tilde{v}^\pi$ and $\hat{v}^\pi$ are exponentially close to each other.


\noindent\textbf{Disentangling the dependency by backwards peeling.} The fictitious estimator technique reduces the problem of calculating the MSE of the MIS estimator to a variance analysis of the fictitious estimator. By recursively applying the law of total variance backwards that peels one item at a time from $\text{Data}_t$, we establish an exact linear decomposition of the variance of the fictitious estimator (Lemma~\ref{lem:var_decomp_fictitious}): 
{\small
	\begin{align*}
\Var[\widetilde{v}^\pi] =  \sum_{h=0}^H \sum_{s_h}  \E\left[\frac{\widetilde{d}_h^\pi(s_h)^2}{n_{s_h}} \mathbf{1}{\left(n_{s_h}\geq \frac{n d_h^{\mu}(s_h)}{(1-\delta)^{-1}}\right)}\right]   \Var_\mu\left[\frac{\pi(a_h^{(1)}|s_h)}{\mu(a_h^{(1)}|s_h)} (V_{h+1}^\pi(s_{h+1}^{(1)}) +  r_h^{(1)})\middle| s_{h}^{(1)}=s_h\right].
\end{align*}
}
Observe that the value function $V_t^\pi$ shows up naturally. 
This novel decomposition can be thought of as a generalization of the celebrated Bellman-equation of variance  \citep{sobel1982variance} in the off-policy, episodic MDP setting with a finite sample and can be of independent interest. 

\noindent\textbf{Characterizing the error propagation in $\tilde{d}_h^\pi(s_h)$.}
Lastly, we bound the error term in the state distribution estimation as follows
$$
\E\left[   \frac{\widetilde{d}_{h}^\pi(s_h)^2}{n_{s_h}}   \mathbf{1}{\left( n_{s_h}\geq \frac{n d_h^{\mu}(s_h)}{(1-\delta)^{-1}}\right)}\right]  \leq    \frac{(1-\delta)^{-1}}{n} \left( \frac{d_{h}^\pi(s_h)^2 }{d_h^\mu(s_h)} + \Var\left[\widetilde{d}_{h}^\pi(s_h)\right]\right),
$$
which reduces the problem to bounding $\Var[\widetilde{d}_{h}^\pi(s_h)]$. We show (in Theorem~\ref{thm:fictitious_error_prop}) that instead of an exponential blow-up as will a concentration-inequality based argument imply, the variance increases at most linearly in $h$:
$
\Var[\widetilde{d}_h^\pi(s_h)] \leq  \frac{2(1-\delta)^{-1} h d_h^\pi(s_h)}{n}.
$
%
The proof uses a novel decomposition of $\Cov(\widetilde{d}_h^\pi)$  (Lemma~\ref{lem:ficticious_cov}), which is derived using a similar backwards peeling argument as before. Finally, Theorem~\ref{thm:main} is established by appropriately choosing $\delta =  O(\sqrt{\log n/n\min_{t,s_t}d_t^\mu(s_t)})$.

Due to space limits, we can only highlight a few key elements of the proof. We invite the readers to check out a more detailed exposition in Appendix~\ref{sec:full_theo}.

\section{Experiments}
\label{sec:exp}

Throughout this section, we present the empirical results to illustrate the comparison among different estimators. We demonstrate the effectiveness of our proposed marginalized estimator by comparing it with different classic estimators on several domains.

The methods we compare in this section are: \textit{direct method} (DM), \textit{importance sampling} (IS), \textit{weighted importance sampling} (WIS), \textit{importance sampling with stationary state distribution} (SSD-IS), and \textit{marginalized importance sampling} (MIS). DM uses the model-based approach to estimate $
  T_t(s_t|s_{t-1},a_{t-1}),
  r_t(s_t,a_t)
$ by enumerating all tuples of $(s_{t-1}, a_{t-1}, s_t)$, IS is the step-wise importance sampling method, WIS uses the step-wise weighted (self-normalized) importance sampling method, SSD-IS denotes the method of importance sampling with stationary state distribution proposed by \citep{liu2018breaking}\footnote{Our implementation of SSD-IS for the discrete state case is described in Appendix~\ref{sec:ssd_is_details}. 
}, and MIS is our proposed marginalized method. Note that our MIS also uses the trick of self-normalization to obtain better performance, but the MIS normalization is different: we normalize the estimate $\widehat d_t^\pi$ to the probability simplex, whereas WIS normalizes the importance weights.
We provide further results by comparing doubly robust estimator, weighted doubly robust estimator, and our estimators in Appendix \ref{sec:exp_app}.
We use logarithmic scales in all figures and include $95\%$ confidence intervals as error bars from 128 runs. Our metric is the relative root mean squared error (\textbf{Relative-RMSE}), which is \textbf{the ratio of RMSE and the true value} $v^\pi$.

\paragraph{Time-invariant MDPs}
\begin{wrapfigure}{r}{0.4\textwidth}
    \centering
    \begin{subfigure}[b]{0.48\linewidth}
        \centering
        \includegraphics[width=\linewidth, page=3]{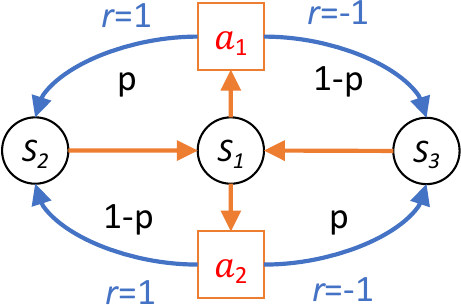}
        \caption{ModelWin}
        \label{fig:ModelWinMDP-main}
    \end{subfigure}
    \begin{subfigure}[b]{0.48\linewidth}
        \centering
        \includegraphics[width=\linewidth, page=4]{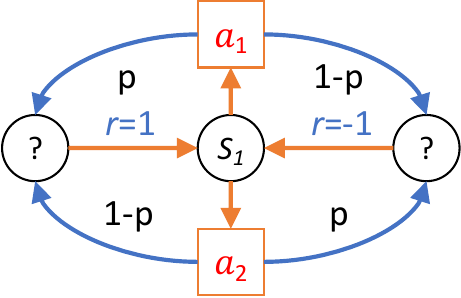}
        \caption{ModelFail}
        \label{fig:ModelFailMDP-main}
    \end{subfigure}
    \caption{MDPs of OPE domains.}
    \label{fig:modelwinfail}
\end{wrapfigure}
We first test our methods on the standard ModelWin and ModelFail models with time-invariant MDPs,  first introduced by \citet{thomas2016data}.
The \textbf{ModelWin} domain simulates a fully observable MDP, depicted in Figure~\ref{fig:ModelWinMDP-main}.
On the other hand, the \textbf{ModelFail} domain (Figure~\ref{fig:ModelFailMDP-main}) simulates a partially observable MDP, where the agent can only tell the difference between $s_1$ and the ``other'' unobservable states.
A detailed description of these two domains can be found in Appendix \ref{sec:exp_app}.
For both problems, the target policy $\pi$ is to always select $a_1$ and $a_2$ with probabilities $0.2$ and $0.8$, respectively, and the behavior policy $\mu$ is a uniform policy.

We provide two types of experiments to show the properties of our marginalized approach. The first kind is with different numbers of episodes, where we use a fixed horizon $H = 50$. The second kind is with different horizons, where we use a fixed number of episodes $n = 1024$.
%
We use MIS only with observable states and the partial trajectories between them. Details about applying MIS with partial observability can be found in Appendix \ref{app:ope_framework}.
While this approach is general in more complex applications,
for ModelFail, the agent always visits $s_1$ at every other step and we can simply replace
$
\frac{\pi(a_{t}^{(i)}|s_{t}^{(i)})}
{\mu(a_{t}^{(i)}|s_{t}^{(i)})}
$
with
$
\frac{\pi(a_{2\tau}^{(i)}|s_{2\tau}^{(i)} = ?)}
{\mu(a_{2\tau}^{(i)}|s_{2\tau}^{(i)} = ?)}
\frac{\pi(a_{2\tau-1}^{(i)}|s_{2\tau-1}^{(i)})}
{\mu(a_{2\tau-1}^{(i)}|s_{2\tau-1}^{(i)})}
$
for $t=2\tau-1$ in \eqref{eq:mis-reward}.

\begin{figure*}[thb]
    \centering
    \begin{subfigure}[b]{1\linewidth}
        \centering
        \includegraphics[width=0.6\linewidth]{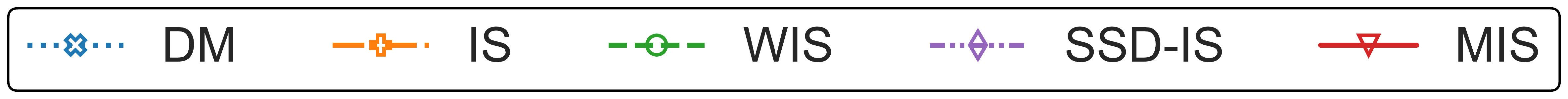}
    \end{subfigure}\\
    \begin{subfigure}[b]{0.2345\linewidth}
        \centering
        \includegraphics[width=\linewidth]{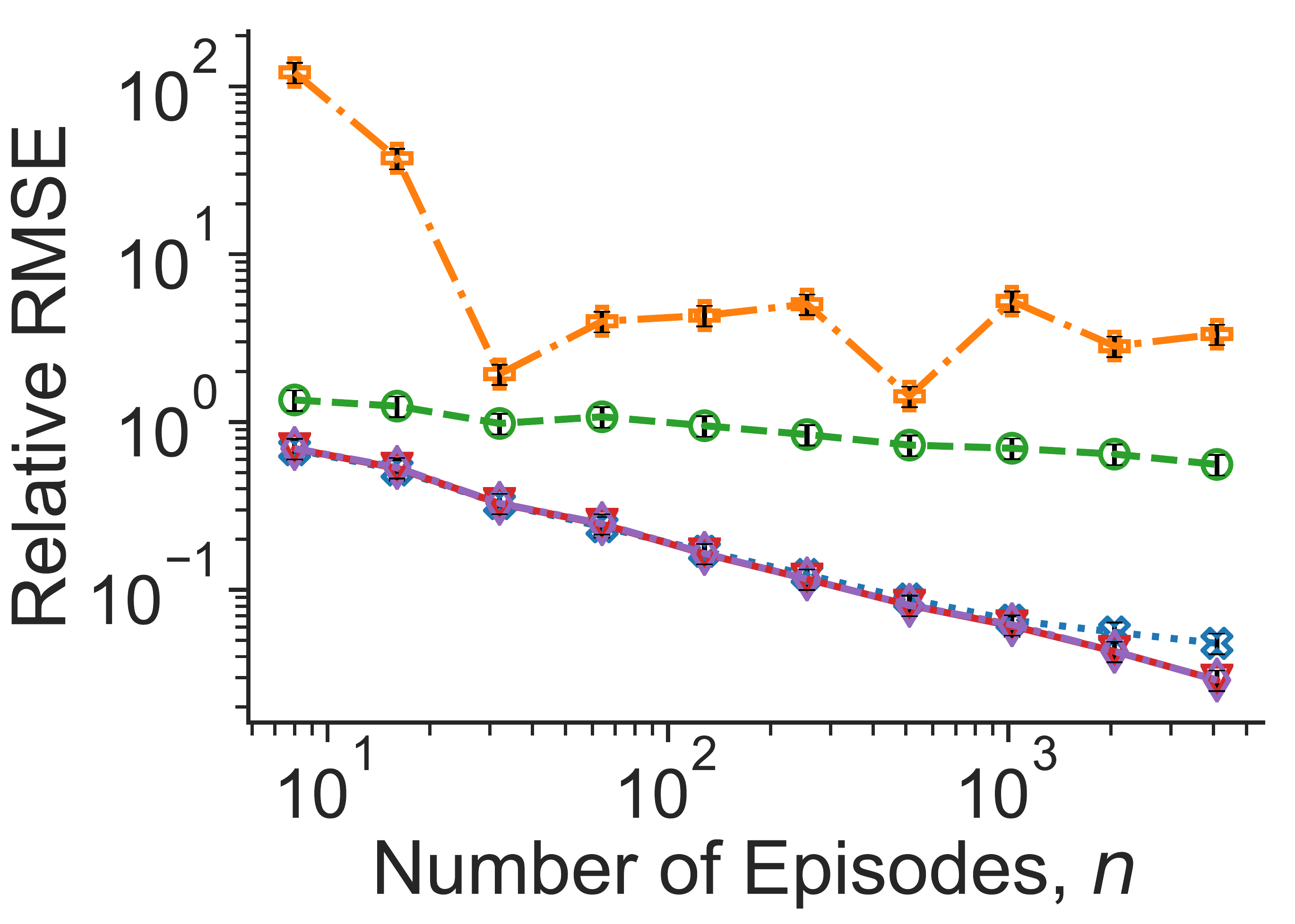}
        \caption{ModelWin with different number of episodes $n$.}
        \label{fig:tinv_ModelWin_n}
    \end{subfigure}
    ~
    \begin{subfigure}[b]{0.2345\linewidth}
        \centering
        \includegraphics[width=\linewidth]{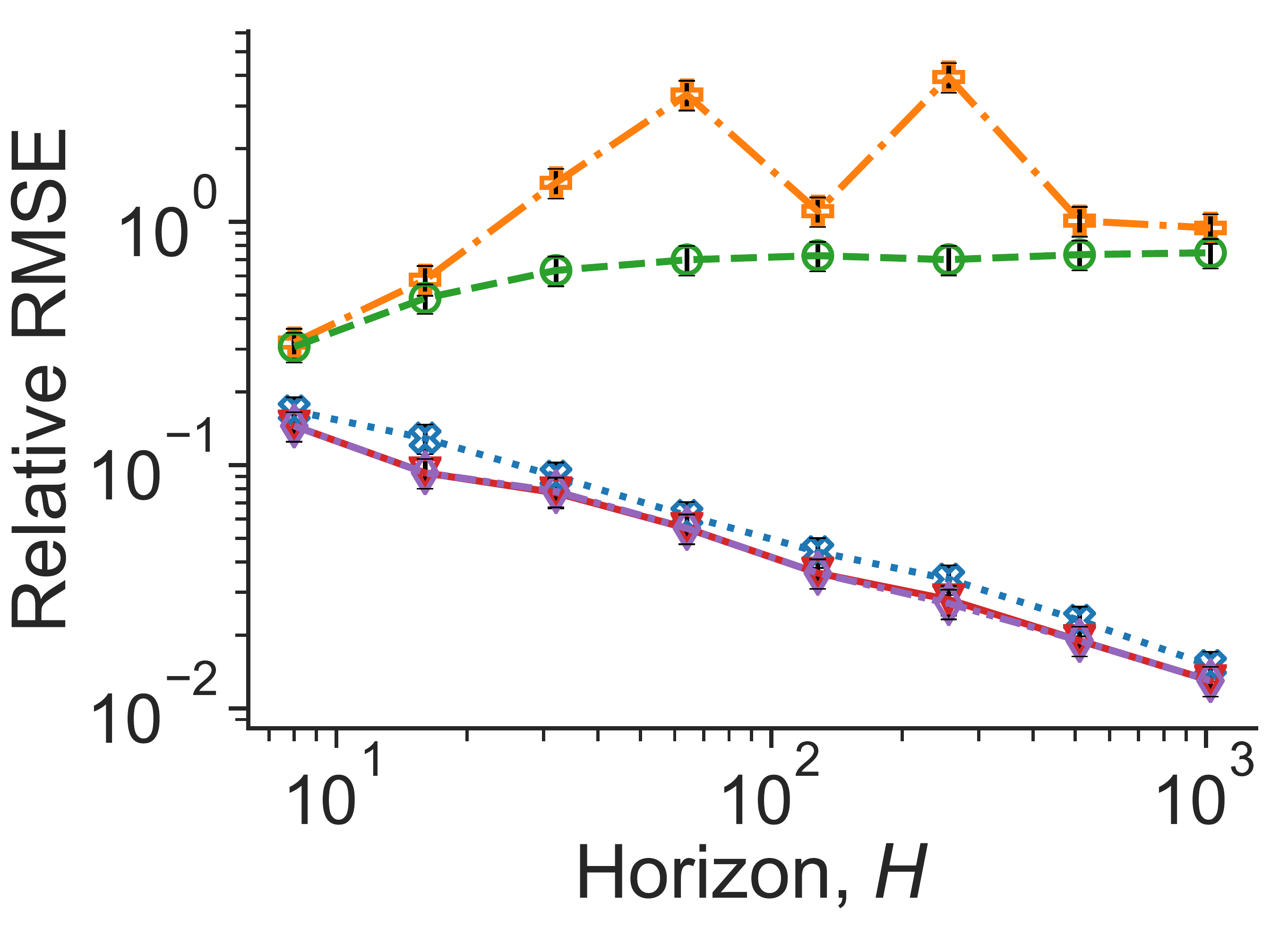}
        \caption{ModelWin with different horizon $H$.}
        \label{fig:tinv_ModelWin_H}
    \end{subfigure}
    ~
    \begin{subfigure}[b]{0.2345\linewidth}
        \centering
        \includegraphics[width=\linewidth]{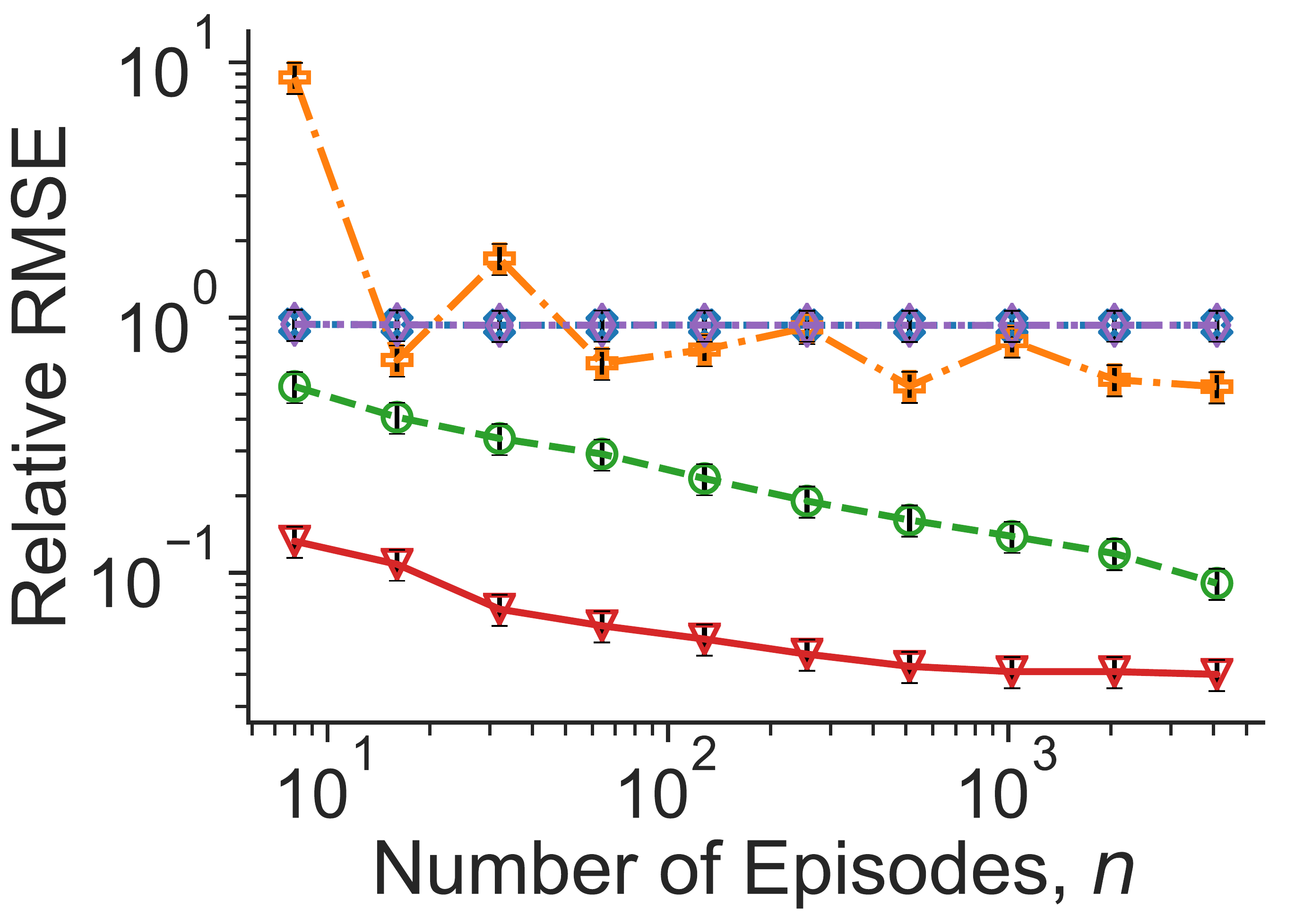}
        \caption{ModelFail with different number of episodes $n$.}
        \label{fig:tinv_ModelFail_n}
    \end{subfigure}
    ~
    \begin{subfigure}[b]{0.2345\linewidth}
        \centering
        \includegraphics[width=\linewidth]{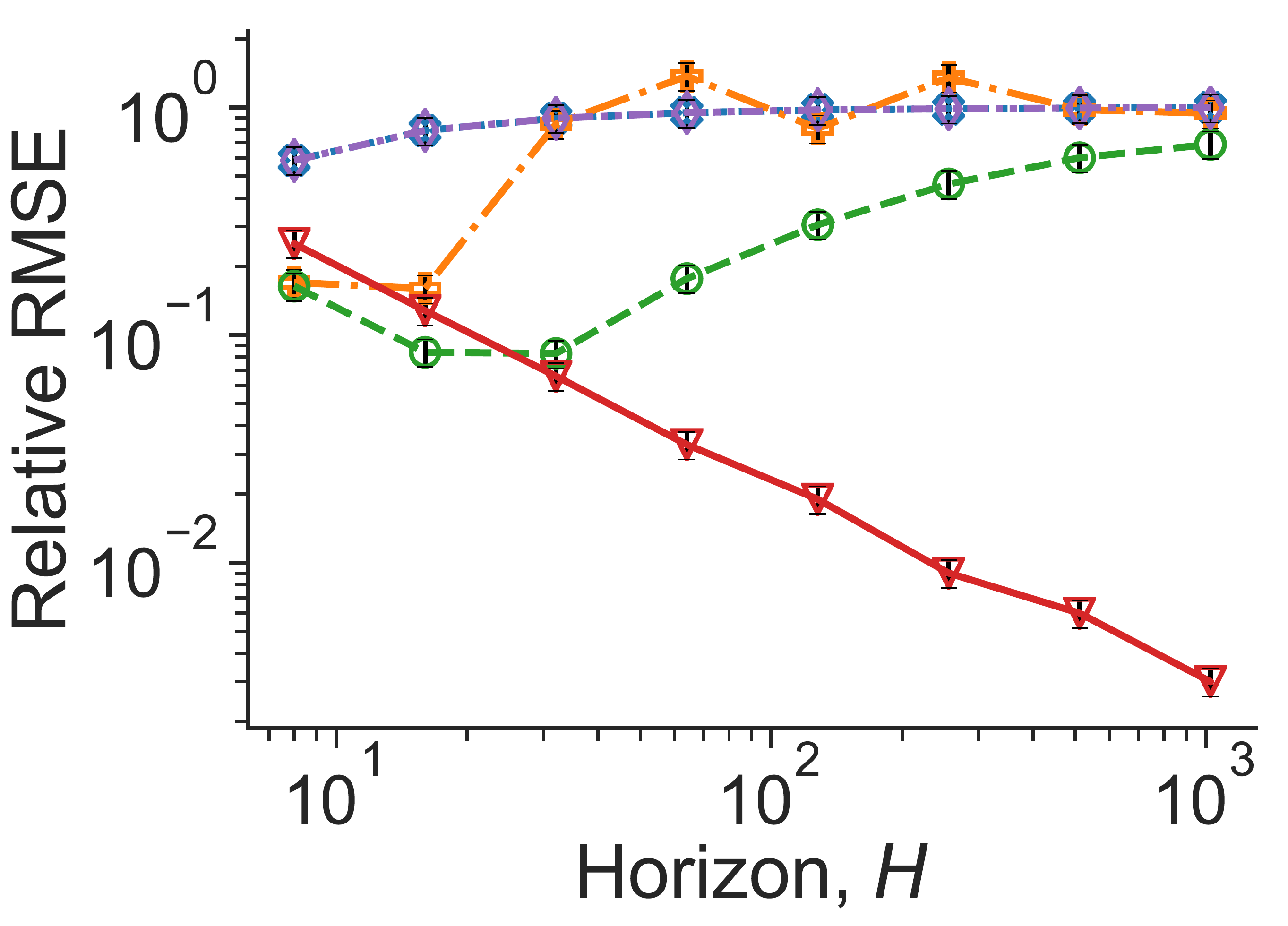}
        \caption{ModelFail with different horizon $H$.}
        \label{fig:tinv_ModelFail_H}
    \end{subfigure}
    \caption{Results on Time-invariant MDPs. MIS matches DM on ModelWin and outperforms IS/WIS on ModelFail, both of which are the best existing methods on their respective domains.}
    \label{fig:modelwinfail_timinv}
\end{figure*}

Figure~\ref{fig:modelwinfail_timinv} shows the results in the time-invariant ModelWin MDP and ModelFail MDP.
The results clearly demonstrate that MIS maintains a polynomial dependence on $H$ and matches the best alternatives such as DM in Figure~\ref{fig:tinv_ModelWin_H} and IS at the beginning of Figure~\ref{fig:tinv_ModelFail_H}.
Notably, the IS in Figure~\ref{fig:tinv_ModelFail_H} reflects a bias-variance trade-off, that its RMSE is smaller at short horizons due to unbiasedness yet larger at long horizons due to high variance.
\paragraph{Time-varying, non-mixing MDPs with continuous actions.}
We also test our approach in simulated MDP environments
where the states are binary, the actions are continuous between [0,1] and the state transition models are time-varying with a finite horizon $H$.
The agent starts at State 1.
At every step, the environment samples a random parameter $p\in[0.5/H, 0.5-0.5/H]$.
Any agent in State 1 will transition to State 0 if and only if it samples an action between $[p-0.5/H, p+0.5/H]$.
On the other hand, State 0 is a sinking state.
The agent collects rewards at State 0 in the latter half of the steps $(t\geq H/2)$.
Thus, the agent wants to transition to State 0, but the transition probability is inversely proportional to the horizon $H$ for uniform action policies.
We pick the behavior policy to be uniform on $[0,1]$ and the target policy to be uniform on $[0,0.5]$ with $95\%$ total probability and $5\%$ chance uniformly distributed on $[0.5,1]$.


\begin{wrapfigure}{r}{0.52\textwidth}
\centering
\includegraphics[width=0.7\linewidth]{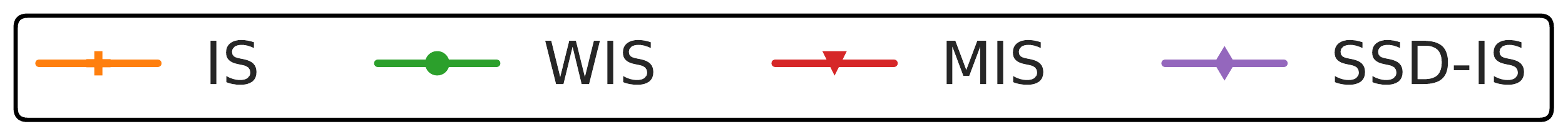}\\
\centering
    \begin{subfigure}[b]{0.48\linewidth}
        \centering
        \includegraphics[width=\linewidth]{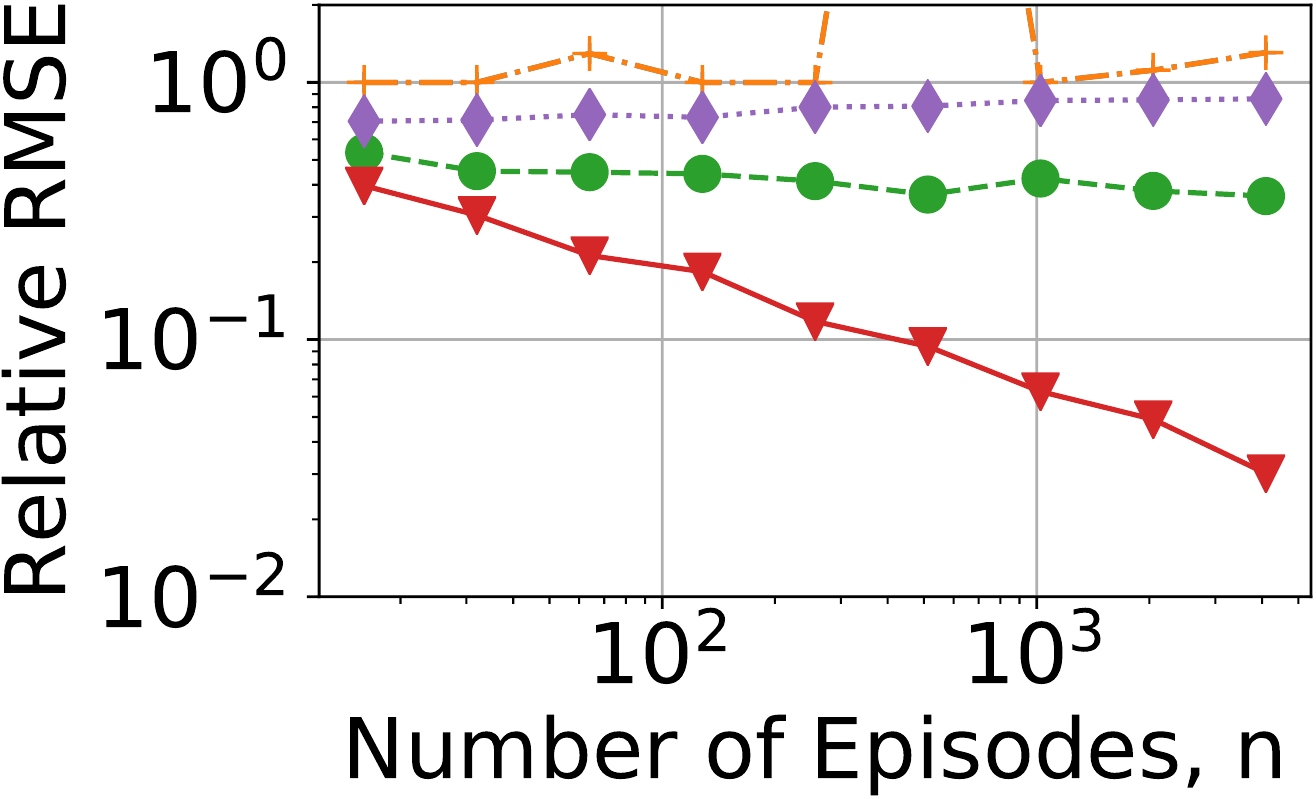}
        \caption{Dependency on $n$}
        \label{fig:tvar_n}
    \end{subfigure}
    ~
    \begin{subfigure}[b]{0.48\linewidth}
        \centering
        \includegraphics[width=\linewidth]{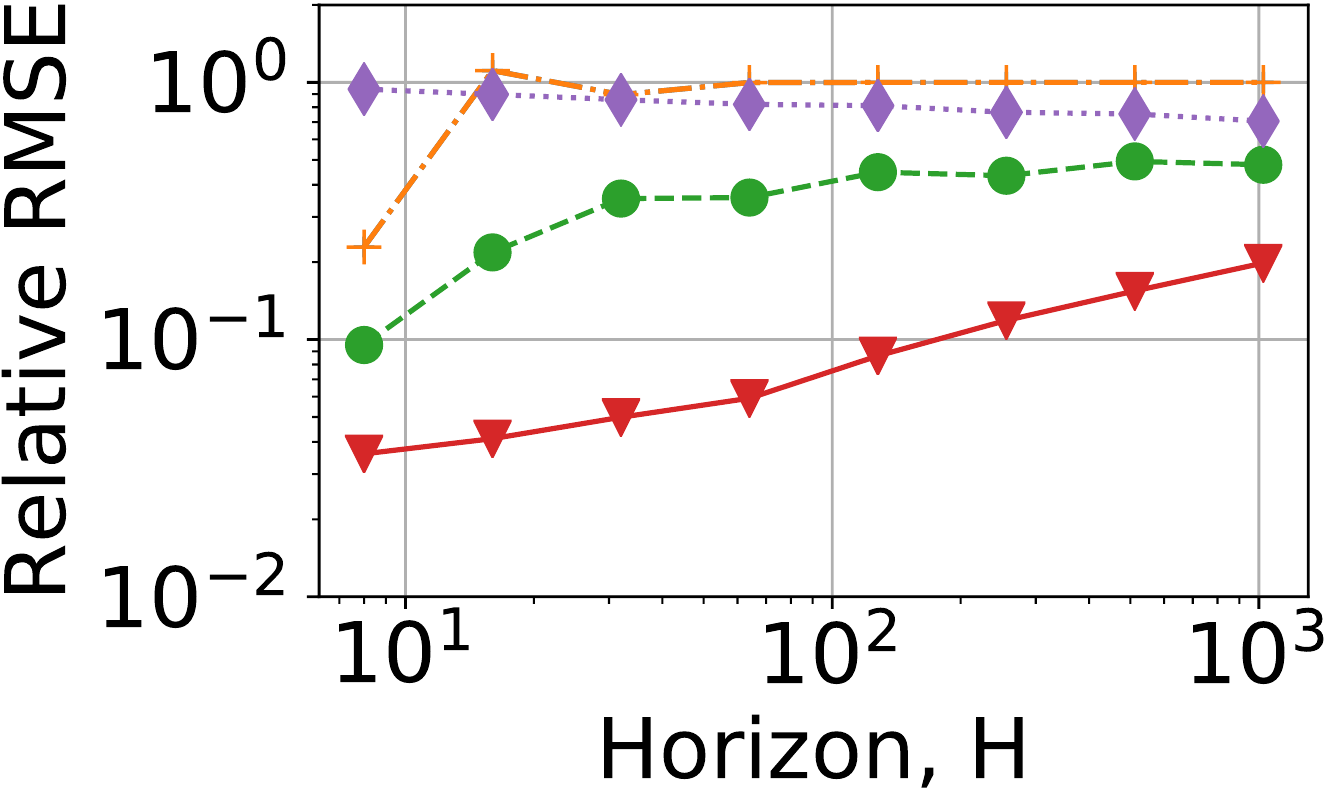}
        \caption{Dependency on $H$}
        \label{fig:tvar_H}
    \end{subfigure}
\caption{Time-varying MDPs}
\label{fig:tvar}
\end{wrapfigure}

Figure~\ref{fig:tvar_n} shows the asymptotic convergence rates of RMSE with respect to the number of episodes, given fixed horizon $H=64$. MIS converges at a $O(\nicefrac{1}{\sqrt{n}})$ rate from the very beginning. In comparison, neither IS or MIS has entered their asymptotic $n^{-1/2}$ regime yet with $n \leq 4,096$. 
SSD-IS does not improve as $n$ gets larger, because the stationary state distribution (a point mass on State $0$) is not a good approximation of the average probability of visiting State 0 for $t \in [H/2,H]$.
We exclude DM because it requires additional model assumptions to apply to continuous action spaces.

Figure~\ref{fig:tvar_H} shows the Relative RMSE dependency in $H$, fixing the number of episodes $n=1024$. We see that as $H$ gets larger, the Relative RMSE scales as $O(\sqrt{H})$ for MIS and stays roughly constant for SSD-IS. Since the true reward $v^\pi \propto H$, the result matches the worst-case bound of a $O(H^3)$ MSE in Corollary~\ref{cor:simple_bound}. SSD-IS saves a factor of $H$ in variance, as it marginalizes over the $H$ steps, but introduces a large bias as we have seen in Figure~\ref{fig:tvar_n}.
IS and WIS worked better for small $H$, but quickly deteriorates as $H$ increases. 
Together with Figure~\ref{fig:tvar_n}, we may conclude that 
In conclusion, MIS is the only method, among the alternatives in this example, that produces a consistent estimator with low variance. 

\paragraph{Mountain Car.}
\begin{figure}[htb]
\centering
\includegraphics[width=0.5\linewidth]{figures/legends.pdf}\\
\centering
\includegraphics[width=0.3\linewidth]{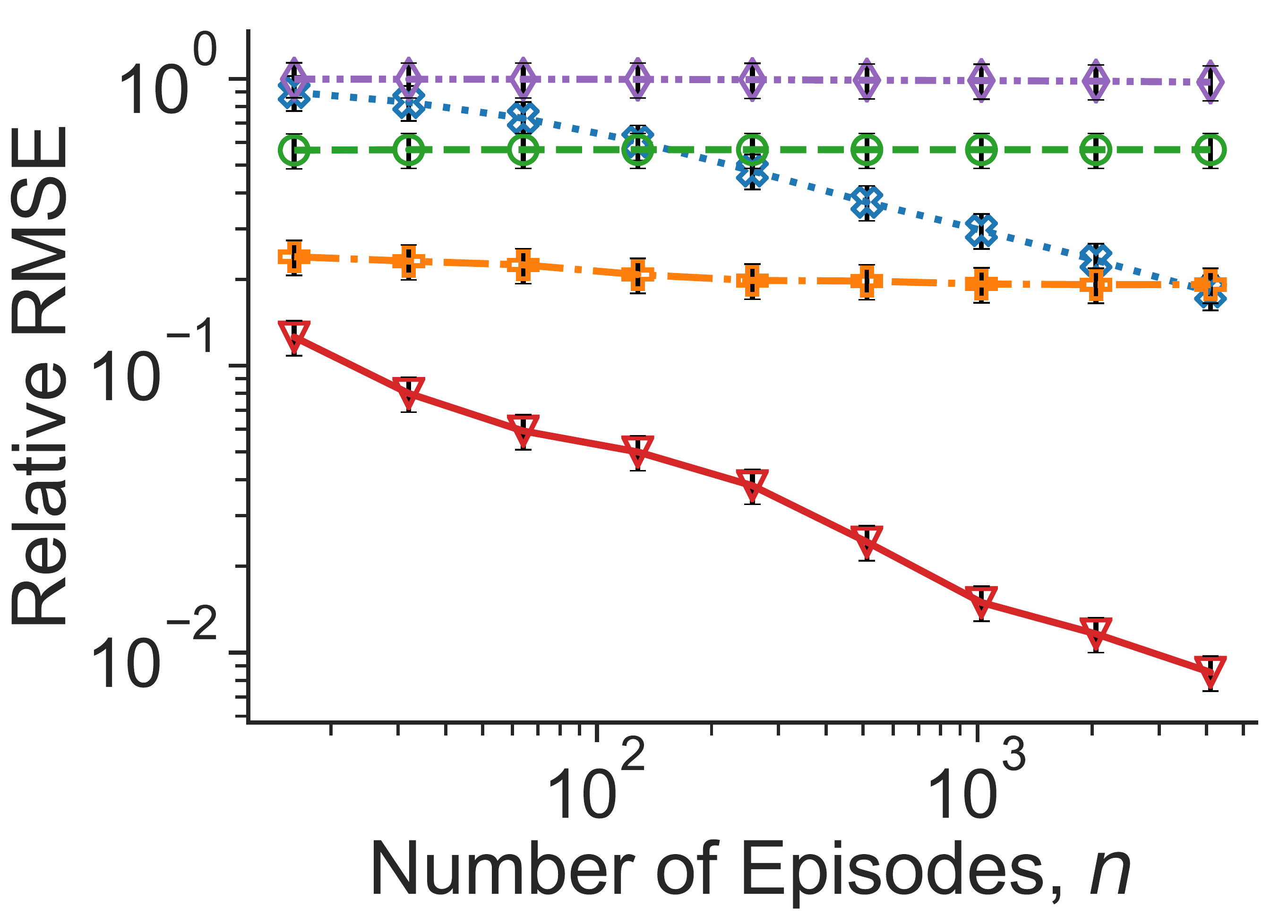}
\caption{Mountain Car with different number of episodes.}
\label{fig:tvar_MountainCar_n}
\end{figure}
Finally, we benchmark our estimator on the Mountain Car domain \citep{singh1996reinforcement}, where an under-powered car drives up a steep valley by ``swinging'' on both sides to gradually build up potential energy. 
%
To construct the stochastic behavior policy $\mu$ and stochastic evaluated policy $\pi$, we first compute the optimal Q-function using Q-learning and use its softmax policy of the optimal Q-function as evaluated policy $\pi$ (with the temperature of $1$). For the behavior policy $\mu$, we also use the softmax policy of the optimal Q-function but set the temperature to $1.25$. Note that this is a finite-horizon MDP with continuous state. We apply MIS by discretizing the state space as in \citep{jiang2016doubly}.


The results
, shown in Figure~\ref{fig:tvar_MountainCar_n}, demonstrate the effectiveness of our approach in a common benchmark control task, where the ability to evaluate under long horizons is required for success. Note that Mountain Car is an episodic environment with a absorbing state, so it is not a setting that SSD-IS is designed for. We include the the detailed description on the experimental setup and discussion on the results in Appendix \ref{sec:exp_app}.

\section{Conclusions}
\label{sec:con}

In this paper, we propose a marginalized importance sampling (MIS) method for the problem of off-policy evaluation in reinforcement learning.
Our approach gets rid of the burden of horizon by using an estimated marginal state distribution of the target policy at every step instead of the cumulative product of importance weights. 

Comparing to the pioneering work of \citet{liu2018breaking} that uses a similar philosophy, this paper focuses on the finite state episodic setting with an potentially infinite action space. We proved the first finite sample error bound for such estimators with polynomial dependence in all parameters. The error bound is tight in that it matches the asymptotic variance of a fictitious estimator that has access to oracle information up to a low-order additive factor. Moreover, it is within a factor of $O(H)$ of the Cramer-Rao lower bound of this problem in \citep{jiang2016doubly}. We conjecture that this additional factor of $H$ is required for any estimators in the \emph{infinite action} setting.

Our experiments demonstrate that the MIS estimator is effective in practice as it achieves substantially better performance than existing approaches in a number of benchmarks.

\section*{Acknowledgement}

The authors thank Yu Bai, Murali Narayanaswamy, Lin F. Yang, Nan Jiang, Phil Thomas, Ying Yang for helpful discussion and Amazon internal review committee for the feedback on an early version of the paper. We also acknowledge the NeurIPS area chair, anonymous reviewers for helpful comments and Ming Yin for carefully proofreading the paper.

YW was supported by a start-up grant from UCSB CS department, NSF-OAC 1934641 and a gift from AWS ML Research Award.

\bibliography{paper}
\bibliographystyle{apalike}

\clearpage

\appendix

\clearpage
\begin{center}
{\LARGE Appendix}
\end{center}

\section{Concentration inequalities}

\begin{lemma}[Multiplicative Chernoff bound \citep{chernoff1952measure} ]\label{lem:chernoff_multiplicative}
	Let $X$ be a Binomial random variable with parameter $p,n$. For any $\delta>0$, we have that 
	$$
	\P[X > (1+\delta)pn] <  \left(\frac{e^\delta}{(1+\delta)^{1+\delta}}\right)^{np}
	$$
	and 
	$$
	\P[X <(1-\delta)pn] < \left( \frac{e^{-\delta}}{(1-\delta)^{1-\delta}} \right)^{np}.  
	$$
\end{lemma}
	A slightly weaker bound that suffices for our propose is the following:
$$
\P[X < (1-\delta)pn] <  e^{-\frac{\delta^2 pn}{2}}
$$
If we take $\delta = \sqrt{\frac{20\log(n)}{pn}}$, 
$$
\P[X < (1-\delta)pn] <  n^{-10}.
$$

\section{Theoretical analysis of the marginalized IS estimator}
\label{sec:full_theo}


Recall that the marginalized IS estimators are of the following form:
$$
\widehat{v}^{\pi}  =  \sum_{t = 1}^H\sum_{s_t}  \widehat{d}_{t}^{\pi}(s_t) \widehat{r}_t^{\pi}(s_t),
$$
where we recursively estimate the state-marginal under the target policy $\pi$ using
$$
\widehat{d}_{t}^{\pi}(s_t)  =  \sum_{s_{t-1}}  \widehat{P}^{\pi}_{t-1,t}(s_t | s_{t-1})  \widehat{d}_{t-1}^{\pi}(s_{t-1}).
$$

We focus on the setting where the number of actions is large and possibly unbounded, in which case, we use importance sampling based estimators of $ \widehat{P}^{\pi}_{t-1,t} $ and $\widehat{r}_t^{\pi}(s_t) $ instead to get bounds that are independent to $A$. Specifically, we use:
$$
\widehat{P}^{\pi}_{t-1}(s_{t} | s_{t-1}) =  \frac{1}{n_{s_{t-1}}} \sum_{i=1}^{n}   \frac{\pi( a_{t-1}^{(i)}| s_{t-1})}{\mu( a_{t-1}^{(i)}| s_{t-1})}  \mathbf{1}(s_{t-1}^{(i)} = s_{t-1}, a_{t-1}^{(i)}, s_t^{(i)} = s_t) .
$$
and
$$\widehat{r}_t^{\pi}(s_t)  =  \frac{1}{n_{s_t}}\sum_{i=1}^n \frac{\pi(a_t^{(i)}|s_t)}{\mu(a_t^{(i)}|s_t)} r_t^{(i)}  \mathbf{1}(s_t^{(i)} = s_t).$$

The main challenge in analyzing these involves finding a way to decompose the error in the face of the complex recursive structure, as well as to deal with the bias of the estimator. 


\paragraph{Constructing a fictitious estimator.}
Our proof makes novel use of a fictitious estimator $\widetilde{v}^\pi$ which uses $\widetilde{d}_t^\pi = \widetilde{P}_{t+1,t}^\pi \widetilde{d}_{t-1}^\pi$ and $\widetilde{r}_t^\pi$ instead of $\widehat{d}_t^\pi = \widehat{P}_{t+1,t}^\pi(\cdot | s_t) \widehat{d}_{t-1}^\pi$ and $\widehat{r}_t^\pi$ in the original estimator $\widehat{v}^\pi$.

To write it down more formally,
$$
\widetilde{v}^\pi :=  \sum_{t=1}^H \sum_{s_t}\widetilde{d}_t^\pi(s_t)  \widetilde{r}_t^\pi(s_t)
$$
where $\widetilde{d}_t^\pi(s_t)$ is constructed recursively using 
$$
\widetilde{d}_t^\pi =  \widetilde{\P}_{t,t-1}^\pi   \widetilde{d}_{t-1}^\pi
$$
as in our regular estimator for $t=2,3,4,...,H$, and $\widetilde{d}_1^\pi  = \widehat{d}_1$.
In particular, 
$$ \widetilde{r}_t^\pi(s_t) =  \begin{cases}
\widehat{r}_t^\pi(s_t)  & \mbox{ if } n_{s_t} \geq n d_t^\mu(s_t)(1-\delta)\\
r_t^\pi(s_t) &\mbox{ otherwise;}
\end{cases}$$ 
and
$$
\widetilde{\P}_{t,t-1}^\pi(\cdot | s_{t-1})  = \begin{cases}
\widehat{\P}_{t,t-1}^\pi    &\mbox{ if }n_{s_{t-1}} \geq n d_t^\mu(s_{t-1})(1-\delta)\\
\P_{t,t-1}^\pi &\mbox{ otherwise.}
\end{cases}
$$
In the above, $0<\delta<1$ is a parameter that we will choose later.

This estimator $\widetilde{v}^\pi$ is fictitious because it is \emph{not implementable} using the data\footnote{It depends on unknown information such as $d_t^\mu$, $\P_{t,t-1}^\pi$, exact conditional expectation of the reward $r_t^\pi$ and so on.}, but it is somewhat easier to work with and behaves essentially the same as our actual estimator $\widehat{v}^\pi$. As a result, we can analyze our estimator through analyzing $\widetilde{v}^\pi$. The following lemma formalizes the idea.

\begin{lemma}\label{lem:fictitious_approximation}
	Let $\widehat{v}^\pi$ be our MIS estimator and $\cP$ be the projection operator to $[0,HR_{\max }]$ and $\widetilde{v}^\pi$ be the unbiased fictitious estimator that we described above with parameter $\delta$. The MSE of the clipped version of our MIS estimator  obeys
	$$
	\E[(\cP \widehat{v}^\pi- v^\pi)^2]  \leq \E[(\widetilde{v}^\pi- v^\pi)^2] + 3H^3SR_{\max}^2 e^{-\frac{\delta^2n \min_{t,s_t}d^{\mu}_t(s_t)}{2}}
	$$
\end{lemma}
\begin{proof}[Proof of Lemma~\ref{lem:fictitious_approximation}]
	Let $E$ denotes the event of $\{\exists t, s_t, \;\text{s.t.}\; n_{s_t} <  n d_t^\mu(s_t)(1-\delta)  \}$.
	Let $\cP_E$ be the \emph{conditional} projection operator that clips the value to $[0,HR_{\max}]$ whenever $E$ is true.
	Note that for any $x\in\R$, we have $\cP (\cP_E x) = \cP x$.  By the non-expansiveness of $\cP$, 
	\begin{align*}
	&\E[(\cP \widehat{v}^\pi  - v^\pi)^2] \leq  \E[(\cP_E \widehat{v}^\pi- v^\pi)^2] =  \E[  (\cP_E \widehat{v}^\pi - \cP_E \widetilde{v}^\pi + \cP_E\widetilde{v}^\pi - v^\pi)^2 ]  \\
	=& \E[(\cP_E \widehat{v}^\pi - \cP_E\widetilde{v}^\pi)^2] +2\E[(\cP_E \widehat{v}^\pi - \cP_E\widetilde{v}^\pi)(\cP_E\widetilde{v}^\pi - v^\pi)]  + \E[(\cP_E \widetilde{v}^\pi - v^\pi)^2]\\
	=& \P[E] \E\big[(\cP_E\widehat{v}^\pi - \cP_E\widetilde{v}^\pi)^2+ 2(\cP_E\widehat{v}^\pi - \cP_E\widetilde{v}^\pi)(\cP_E\widetilde{v}^\pi - v^\pi) \big|E\big] + \P[E^c] \cdot 0  + \E[(\cP_E \widetilde{v}^\pi - \cP_E v^\pi)^2]\\
	\leq& 3 \P[E]H^2R_{\max}^2+ \E[(\widetilde{v}^\pi - v^\pi)^2]. 
	\end{align*}
	The third line is by the law of total expectation and the fact that whenever $E$ is not true, $\widehat{v}^\pi = \widetilde{v}^\pi$.  The last line uses the fact that $\cP_E \widehat{v}^\pi,\cP_E \widetilde{v}^\pi, v^\pi$ are all within $[0,HR_{\max}]$ when conditioning on $E$ as well as the non-expansiveness of the projection operator which implies that 
	$$
	\E[(\cP_E(\widetilde{v}^\pi - v^\pi))^2] \leq \E[(\widetilde{v}^\pi - v^\pi)^2].
	$$
	It remains to bound	$
	\P[E]
	$.
	By the multiplicative Chernoff bound (Lemma~\ref{lem:chernoff_multiplicative} in the Appendix) we get that
		$$
	\P\left[ n_{s_t} < n d^{\mu}_t(s_t)(1-\delta) \right] \leq  e^{-\frac{\delta^2 n d^{\mu}_t(s_t)}{2}} 
	$$
	By a union bound over each $t$ and $s_t$, we have
	\begin{align*}
	\P[E] &\leq \sum_{t}\sum_{s_t} \P[n_{s_t,t} < n d^{\mu}_t(s_t)(1-\delta)] \leq HS e^{-\frac{\delta^2 n \min_{t,s_t} d^{\mu}_t(s_t)}{2}}
	\end{align*}
	as stated.
\end{proof}

Lemma~\ref{lem:fictitious_approximation} establishes that  when $n \geq \frac{\mathrm{polylog}(S,H,n)}{\min_{t,s_t}d^\mu_{t}(s_t)}$, we can bound the MSE of a projected version of our estimator using the MSE of the fictitious estimator. The projection to $[0,HR_{\max}]$ is a post-processing that we needed in our proof for technical reasons, and we know that $\E[(\cP \widehat{v}^\pi - v^\pi)^2 ]\leq \E[(\widehat{v}^\pi - v^\pi)^2]$ so it only improves the performance.

\paragraph{Properties of the Fictitious Estimator.}
Now let us prove that $\widetilde{v}^\pi$ is unbiased and also analyze its variance. Recall that the estimator is the following:
$$
\widetilde{v}^{\pi}  =  \sum_{t = 1}^H\sum_{s_t}  \widetilde{d}_{t}^{\pi}(s_t)\widetilde{r}_t^{\pi}(s_t) = \sum_{t=1}^H  \langle\widetilde{d}_{t}^{\pi},  \widetilde{r}_t^{\pi} \rangle 
$$
where we denote quantities $\widetilde{d}_{t}^{\pi},\widetilde{r}_t^{\pi}$ in vector forms in $\R^{S}$.  

In the remainder of this section, we will use $E_t$ as a short hand to denote the event such that $\{n_{s_t} \geq nd_t^\mu(s_t)(1-\delta)\}$, and $\mathbf{1}(E_t)$ be the corresponding indicator function.

\begin{lemma}[Unbiasedness of $\widetilde{v}^\pi$]\label{lem:unbiasedness_fictitious}
	$\E[\widetilde{v}^\pi] = v^\pi$ for all $\delta<1$.
\end{lemma}
\begin{proof}[Proof of Lemma~\ref{lem:unbiasedness_fictitious}]
	The idea of the proof is to recursively apply the Law of Total Expectation backwards from the last round by taking conditional expectations. For simplicity of the proof we will denote 
	$$\text{Data}_t := \left\{s_{1:{t}}^{(i)}, a_{1:{t-1}}^{(i)},r_{1:{t-1}}^{(i)} \right\}_{i=1}^n.$$
		Also, in the base case, let's denote $\text{Data}_1 := \left\{s_{1}^{(i)}\right\}_{i=1}^n$ and that 
	$
	r_{t}^\pi(s_t) := \E_\pi[ r_{t}^{(1)} |  s_t^{(1)} = s_t ]
	$
	
	We first making a few observations that will be useful in the arguments that follow. Firstly, $\widetilde{d}_t^\pi$ and $\widetilde{r}_{t-1}^\pi$ are deterministic given $\text{Data}_t$.
	Secondly,  
	$$
	\E[\widetilde{P}_{t,t-1}^\pi |  \text{Data}_{t-1}] =  P_{t,t-1}^\pi,\quad \text{ and }\quad
	\E[\widetilde{r}_{t}^\pi | \text{Data}_{t}] =  r_t^\pi.
	$$
	These observations are true for all $t=1,...,H$. To see the unbiasedness of the conditional expectation, note that when $n_{s_t}\geq n d_t^\mu(s_{t})(1-\delta)$, the estimators are just empirical mean, which are unbiased and when $n_{s_t} < n d_t^\mu(s_{t})(1-\delta)$, we also have an unbiased estimator by the construction of the fictitious estimator. For all $\delta<1$, the case $n_{s_t}=0$ is ruled out.Thirdly, we write down the standard Bellman equation for policy $\pi$
	$$
	V_h(s_h) =   r_{h}^\pi(s_h) + \sum_{s_{h+1}} P_{h+1,h}^\pi(s_{h+1} | s_h) V_{h+1}(s_{h+1}).
	$$
	where $V_h(s_h) := \E_\pi\left[\sum_{t=h}^H r_t^{(1)}\middle|  s_t^{(1)} = s_h \right]$ or in a matrix form
	$$
		V_h =   r_{h}^\pi  + [P_{h+1,h}^\pi]^T V_{h+1}.
	$$
	These observations together allow us to write the following recursion:
	\begin{align*}
&\E\left[\langle\widetilde{d}_{h}^{\pi},  V_h^{\pi}\rangle +  \sum_{t=1}^{h-1}  \langle\widetilde{d}_{t}^{\pi},  \widetilde{r}_t^{\pi}\rangle  \middle|  \text{Data}_{h-1} \right]  \\
=&  \langle \E[\widetilde{P}_{h,h-1}^\pi | \text{Data}_{h-1}]\widetilde{d}_{h-1}^{\pi},  V_h^{\pi}\rangle +\langle\widetilde{d}_{h-1}^{\pi},  \E\left[ \widetilde{r}_{h-1}^{\pi}\middle| \text{Data}_{h-1}\right]\rangle +  \sum_{t=1}^{h-2}  \langle\widetilde{d}_{t}^{\pi},  \widetilde{r}_t^{\pi}\rangle     \\
=& \langle \widetilde{d}_{h-1}^\pi, [P_{h,h-1}^\pi]^{T} V_h^{\pi} + r_{h-1}^\pi \rangle + \sum_{t=1}^{h-2}  \langle\widetilde{d}_{t}^{\pi},  \widetilde{r}_t^{\pi}\rangle   \\
\explain{=}{\text{Bellman equation}}& \langle \widetilde{d}_{h-1}^\pi, V_{h-1}^\pi \rangle + \sum_{t=1}^{h-2}  \langle\widetilde{d}_{t}^{\pi},  \widetilde{r}_t^{\pi}\rangle .
\end{align*}
Finally, by taking (full) expectation and chaining the above recursions together, we get
\begin{align*}
\E\left[  \sum_{t=1}^H \langle\widetilde{d}_{t}^{\pi},  \widetilde{r}_t^{\pi}\rangle \right] &=  \E\left[ \langle \widetilde{d}_{H}^{\pi} ,V^\pi_H\rangle +  \sum_{t=1}^{H-1} \langle\widetilde{d}_{t}^{\pi},  \widetilde{r}_t^{\pi}\rangle  \right]\\
&=\E\left[ \langle \widetilde{d}_{H-1}^{\pi} ,V^\pi_{H-1}\rangle +  \sum_{t=1}^{H-2} \langle\widetilde{d}_{t}^{\pi},  \widetilde{r}_t^{\pi}\rangle  \right]\\
&= \hdots \\
&=\E\left[  \langle \widetilde{d}_{1}^{\pi} ,V^\pi_{1}\rangle   \right] =v^\pi,
\end{align*}
which concludes the proof.
\end{proof}
Now let's tackle the variance of the fictitious estimator.
\begin{lemma}[Variance decomposition] \label{lem:var_decomp_fictitious}
		\begin{align*}
	\Var[\widetilde{v}^\pi] =& \frac{\Var[V_{1}^\pi(s_1^{(1)})]}{n} 
	\\&+  \sum_{h=1}^H \sum_{s_h}  \E\left[\frac{\widetilde{d}_h^\pi(s_h)^2}{n_{s_h}} \mathbf{1}(E_h)\right]    \Var_{\mu}\left[\frac{\pi(a_h^{(1)}|s_h)}{\mu(a_h^{(1)}|s_h)} (V_{h+1}^\pi(s_{h+1}^{(1)}) +  r_h^{(1)})\middle| s_{h}^{(1)}=s_h\right].
	\end{align*}
		where
		$V_t^\pi(s_t)$ denotes the value function under $\pi$ which satisfies the Bellman equation 
	$$ V_t^\pi(s_t)=  r_t^\pi(s_t)  + \sum_{s_{t+1}} P^\pi_t(s_{t+1}| s_t)   V_{t+1}^\pi(s_{t+1}).$$
\end{lemma}
\begin{remark}
The decomposition of variance is very interpretable. The first part of the variance is coming from estimating the initial state. The second part is coming from the conditional variance of estimating $P_{t+1,t}^{\pi}(s_t)$ and $r_t^\pi(s_t)$ using importance sampling over $a_t$.
\end{remark}

\begin{proof}[Proof of Lemma~\ref{lem:var_decomp_fictitious}]
	The proof uses a peeling argument that recursively applies the law of total variance from the last time point backwards.
	
	The key of the argument relies upon the following identity that holds for all $h=1,...,H-1$.
	\begin{equation}
 	\label{eq:var_recursion_ficticious}
	\begin{aligned}
	\Var\left[  \langle \widetilde{d}_{h+1}^\pi, V_{h+1}^\pi\rangle +  \sum_{t=1}^h \langle\widetilde{d}_{t}^{\pi},  \widetilde{r}_t^{\pi} \rangle\right] =&  \E\left[ \Var\left[  \langle \widetilde{d}_{h+1}^\pi, V_{h+1}^\pi\rangle  +  \langle \widetilde{d}_h^\pi, \widetilde{r}_h^\pi\rangle \middle| \text{Data}_{h} \right] \right] \\
	&+ \Var\left[  \langle \widetilde{d}_{h}^\pi, V_{h}^\pi\rangle +  \sum_{t=1}^{h-1} \langle\widetilde{d}_{t}^{\pi},  \widetilde{r}_t^{\pi} \rangle \right].
	\end{aligned}
	\end{equation}
	Note that in \eqref{eq:var_recursion_ficticious}, when we condition on $\text{Data}_h$, $\widetilde{d}_{h}^\pi$ is fixed. 
	Also, $\widetilde{P}_{h+1,h}(\cdot,s_h)$ and $\widetilde{r}_h^\pi(s_h)$ for each $s_h$  are conditionally independent given $\text{Data}_h$, since $\text{Data}_h$ partitions the $n$ episodes into $S$ disjoint sets according to the states  $s_h^{(i)}$ at time $h$. These observations imply that	
	\begin{align}
	&\E\left[ \Var\left[  \langle \widetilde{d}_{h+1}^\pi, V_{h+1}^\pi\rangle  +  \langle \widetilde{d}_h^\pi, \widetilde{r}_h^\pi\rangle \middle| \text{Data}_{h} \right] \right] \\
	=&\E\left[ \sum_{s_h} \Var\left[  \widetilde{d}_{h}^\pi(s_h)\langle\widetilde{P}_{h+1,h}(\cdot,s_h) , V_{h+1}^\pi\rangle  +  \widetilde{d}_h^\pi(s_h)\cdot \widetilde{r}_h^\pi(s_h) \middle| \text{Data}_{h} \right] \right] \\
	=&\E\left[ \sum_{s_h} \mathbf{1}(E_h)  \Var\left[  \widetilde{d}_{h}^\pi(s_h)\langle\widetilde{P}_{h+1,h}(\cdot,s_h) , V_{h+1}^\pi\rangle  +  \widetilde{d}_h^\pi(s_h)\cdot \widetilde{r}_h^\pi(s_h) \middle| \text{Data}_{h} \right] \right] \\
	=&\E\left[ \sum_{s_h} \mathbf{1}(E_h)   \Var\left[  \left\langle  \frac{\widetilde{d}_h^\pi(s_h)}{n_{s_h}}\sum_{i | s_h^{(i)}=s_h}\frac{\pi(a_h^{(i)}|s_h)}{\mu(a_h^{(i)}|s_h)}  \mathbf{e}_{s_{h+1}^{(i)}},  V_{h+1}^\pi\right\rangle +  \frac{\widetilde{d}_h^\pi(s_h)}{n_{s_h}}\sum_{i  | s_h^{(i)}=s_h } \frac{\pi(a_h^{(i)}|s_h)}{\mu(a_h^{(i)}|s_h)}  r_h^{(i)} \middle| \text{Data}_{h} \right] \right]\\
	=&\E\left[ \sum_{s_h}  \widetilde{d}_h^\pi(s_h)^2\mathbf{1}(E_h)  \Var\left[  \frac{1}{n_{s_h}}\sum_{i  | s_h^{(i)}=s_h }  \frac{\pi(a_h^{(i)}|s_h)}{\mu(a_h^{(i)}|s_h)} (V_{h+1}^\pi(s_{h+1}^{(i)}) +  r_h^{(i)})  \middle| \text{Data}_{h} \right] \right]\\
	=&\sum_{s_h}  \E\left[\frac{\widetilde{d}_h^\pi(s_h)^2}{n_{s_h}} \mathbf{1}(E_h)\right]    \Var\left[\frac{\pi(a_h^{(1)}|s_h)}{\mu(a_h^{(1)}|s_h)} (V_{h+1}^\pi(s_{h+1}^{(1)}) +  r_h^{(1)})\middle| s_{h}^{(1)}=s_h\right].\label{eq:var_recursion_per_step}
	\end{align}
	The second line uses the conditional independence we mentioned above. The third line uses that when $n_{s_h} < nd_h^{\mu}(s_h)$, the conditional variance is $0$. The fourth and fifth line apply the definition of the importance sampling estimators and finally the last line uses that the episodes are iid.

	Apply \eqref{eq:var_recursion_ficticious} recursively 
	\begin{align*}
	\Var[\widetilde{v}^{\pi} ]  = & \E \Var[\widetilde{v}^{\pi} |  \text{Data}_H ]  + \Var[ \E[ \widetilde{v}^{\pi} | \text{Data}_H]] \\
	=& \E\left[ \Var[ \langle\widetilde{d}_{H}^{\pi},  \widetilde{r}_H^{\pi} \rangle   |  \text{Data}_H ]\right]    +  \Var[  \E[  \langle\widetilde{d}_{H}^{\pi},  \widetilde{r}_H^{\pi} \rangle  | \text{Data}_{H}]  + \sum_{t=1}^{H-1} \langle\widetilde{d}_{t}^{\pi},  \widetilde{r}_t^{\pi} \rangle  ] \\
	=&\E\left[ \Var[ \langle\widetilde{d}_{H}^{\pi},  \widetilde{r}_H^{\pi} \rangle   |  \text{Data}_H ]\right]   +  \Var[  \langle\widetilde{d}_{H}^{\pi},  r_H^{\pi} \rangle  + \sum_{t=1}^{H-1} \langle\widetilde{d}_{t}^{\pi},  \widetilde{r}_t^{\pi} \rangle   ] \\
	=&\E\left[ \Var[ \langle\widetilde{d}_{H}^{\pi},  \widetilde{r}_H^{\pi} \rangle   |  \text{Data}_H ]\right]    +  \Var[  \langle\widetilde{d}_{H}^{\pi},  V_H^{\pi} \rangle  + \sum_{t=1}^{H-1} \langle\widetilde{d}_{t}^{\pi},  \widetilde{r}_t^{\pi} \rangle   ] \\
	=&\E\left[ \Var[ \langle\widetilde{d}_{H}^{\pi},  \widetilde{r}_H^{\pi} \rangle   |  \text{Data}_H ]\right]  + \E\left[ \Var\left[  \langle \widetilde{d}_{H}^\pi, V_{H}^\pi\rangle  +  \langle \widetilde{d}_{H-1}^\pi, \widetilde{r}_{H-1}^\pi\rangle \middle| \text{Data}_{H-1} \right] \right] \\
	&+ \Var\left[  \langle \widetilde{d}_{H-1}^\pi, V_{H-1}^\pi\rangle +  \sum_{t=1}^{H-2} \langle\widetilde{d}_{t}^{\pi},  \widetilde{r}_t^{\pi} \rangle \right]\\
	=&\E\left[ \Var[ \langle\widetilde{d}_{H}^{\pi},  \widetilde{r}_H^{\pi} \rangle   |  \text{Data}_H ]\right]  +\sum_{h=H-1}^H \E\left[ \Var\left[  \langle \widetilde{d}_{h}^\pi, V_{h}^\pi\rangle  +  \langle \widetilde{d}_{h-1}^\pi, \widetilde{r}_{h-1}^\pi\rangle \middle| \text{Data}_{h-1} \right] \right] \\
	&+\Var\left[  \langle \widetilde{d}_{H-2}^\pi, V_{H-2}^\pi\rangle +  \sum_{t=1}^{H-3} \langle\widetilde{d}_{t}^{\pi},  \widetilde{r}_t^{\pi} \rangle \right]\\
	=&\E\left[ \Var[ \langle\widetilde{d}_{H}^{\pi},  \widetilde{r}_H^{\pi} \rangle   |  \text{Data}_H ]\right]  +\sum_{h=2}^H \E\left[ \Var\left[  \langle \widetilde{d}_{h}^\pi, V_{h}^\pi\rangle  +  \langle \widetilde{d}_{h-1}^\pi, \widetilde{r}_{h-1}^\pi\rangle \middle| \text{Data}_{h-1} \right] \right] +\Var\left[  \langle \widetilde{d}_{1}^\pi, V_{1}^\pi\rangle \right]\\
	\end{align*}
	Use the boundary condition $V_{H+1}\equiv 0$ as stated in the theorem and apply \eqref{eq:var_recursion_per_step}, we get that
	\begin{align*}
	\Var[\widetilde{v}^\pi]  =& \frac{\Var[V_{1}^\pi(s_1^{(1)})]}{n} 
	\\&+  \sum_{h=1}^H \sum_{s_h}  \E\left[\frac{\widetilde{d}_h^\pi(s_h)^2}{n_{s_h}} \mathbf{1}(E_h)\right]    \Var\left[\frac{\pi(a_h^{(1)}|s_h)}{\mu(a_h^{(1)}|s_h)} (V_{h+1}^\pi(s_{h+1}^{(1)}) +  r_h^{(1)})\middle| s_{h}^{(1)}=s_h\right].
	\end{align*}
	This completes the proof.
\end{proof}

\paragraph{Bounding the importance weights}
It remains to show that for all $h,s_h$, 
$$\E\left[   \frac{\widetilde{d}_{h}^\pi(s_h)^2}{n_{s_h}}   \mathbf{1}(E_h)\right] \approx \frac{d_{h}^\pi(s_h)^2}{n d_{h}^{\mu}(s_h)}.$$  

By the non-negativity of $\widetilde{d}_{h}^\pi(s_h)^2$
\begin{equation}\label{eq:reduction_to_variance}
\E\left[   \frac{\widetilde{d}_{h}^\pi(s_h)^2}{n_{s_h}}   \mathbf{1}(E_h)\right]  \leq  \frac{(1-\delta)^{-1}}{ n d_h^\mu(s_h)} \E\left[   \widetilde{d}_{h}^\pi(s_h)^2\right]  =   \frac{(1-\delta)^{-1}}{ n d_h^\mu(s_h)} (d_{h}^\pi(s_h)^2  + \Var[\widetilde{d}_{h}^\pi(s_h)]).
\end{equation}
where the last identity is true because $\widetilde{d}^\pi_h$ is an unbiased estimator of $d_{h}^\pi(s_h)$ as the following lemma establishes.
\begin{lemma}[Unbiasedness of $\widetilde{d}^\pi_h$]
\label{lem:unbiasedness}
	For all $h=1,...,H$, the fictitious state marginal estimators are unbiased, that is,
	$$
	\E[\widetilde{d}^\pi_h] =  d^\pi_h. 
	$$
\end{lemma}
\begin{proof}[Proof of Lemma~\ref{lem:unbiasedness}]
	Recall the recursive relationship by construction
	$$
	\widetilde{d}^\pi_h =    \widetilde{\P}_{h,h-1}^\pi  \widetilde{d}^\pi_{h-1}.
	$$
	We will prove by induction on $h$. First, take the base case $h=1$:
	$\E[\widetilde{d}^\pi_1] =  \E[\widehat{d}^\pi_1] = d^\pi_1$.
	Now if $\E[\widetilde{d}^\pi_{h-1}] =  d^\pi_{h-1}$, then by the law of total expectation:
	\begin{align*}
	\E[\widetilde{d}^\pi_h]  &=  \E\left[  \E[ \widetilde{\P}_{h,h-1}^\pi \widetilde{d}^\pi_{h-1} |  \text{Data}_{h-1}]\right]\\
	&=   \P_{h,h-1}^\pi  \E\left[\widetilde{d}^\pi_{h-1}\right] = \P_{h,h-1}^\pi d^\pi_{h-1} =  d^\pi_h.
	\end{align*}
	This completes the proof for all $h$.
\end{proof}

So the problem reduces to bounding $\Var[\widetilde{d}_{h}^\pi(s_h)]$. We will prove something more useful by bounding the covariance matrix of $\widetilde{d}_{h}^\pi(s_h)$ in semidefinite ordering.


\begin{lemma}[Covariance of $\widetilde{d}_h^\pi$]\label{lem:ficticious_cov}
\begin{align*}
& \Cov(\widetilde{d}^\pi_h)  \\
\preceq&  \frac{(1-\delta)^{-1}}{n}  \sum_{t=2}^h \P^\pi_{h,t} \diag\left[\sum_{s_{t-1}}\frac{d_{t-1}^\pi(s_{t-1})^2 + \Var(\widetilde{d}_{t-1}^\pi(s_{t-1}))}{d_{t-1}^\mu(s_{t-1})}\sum_{a_{t-1}} \frac{\pi(a_{t-1}| s_{t-1})^2}{\mu(a_{h-1}| s_{t-1})}  \P_{t,t-1}(\cdot|s_{t-1},a_{t-1})\right]   \left[\P^\pi_{h,t}\right]^T \\
&+   \frac{1}{n} \P_{h,1}^\pi \diag\left[  d^\pi_1\right] [\P_{h,1}^\pi ]^T.
\end{align*}
where $\P^\pi_{h,t}  =  \P^\pi_{h,h-1}\cdot \P^\pi_{h-1,h-2}\cdot ... \cdot\P^\pi_{t+1,t}$ --- the transition matrices under policy $\pi$ from time $t$ to $h$ (define $\P_{h,h}^\pi := I$).
\end{lemma}
Before proving the result, let us connect it to what we need in \eqref{eq:reduction_to_variance}.
\begin{corollary}\label{cor:variance}
	For $h=1$, we have:
	$$\Var[\widetilde{d}_1^\pi(s_1)] = \frac{1}{n}(d_h^\pi(s_1) -  d_h^\pi(s_1)^2).$$
	
	For $h= 2,3,...,H$, we have:
	$$
	\Var[\widetilde{d}_h^\pi(s_h)]  \leq \frac{(1-\delta)^{-1}}{n} \sum_{t=2}^h \sum_{s_t}  \P_{h,t}^{\pi}( s_h | s_t )^2 \varrho(s_t)   + \frac{1}{n}\sum_{s_1} \P_{h,1}^{\pi}( s_h | s_1)^2 d_{1}(s_1)
	$$
	where 
	$\varrho(s_t) := \sum_{s_{t-1}} \left(\frac{d_{t-1}^\pi(s_{t-1})^2 + \Var(\widetilde{d}_{t-1}^\pi(s_{t-1}))}{d_{t-1}^\mu(s_{t-1})} \sum_{a_{t-1}}\frac{\pi(a_{t-1}|s_{t-1})^2}{\mu(a_{t-1}|s_{t-1})}   \P_{t,t-1}(s_t|s_{t-1},a_{t-1})\right).$
\end{corollary}
Note that we have $\Var[\widetilde{d}_{t-1}^\pi(s_{t-1})]$ on the RHS of the equation, which suggests that we in fact need to recursively apply our bounds from $h=1$ to obtain the overall bound. 
\begin{theorem}[Error propagation]\label{thm:fictitious_error_prop}
	Let $\tau_a := \max_{t,s_t,a_t}\frac{\pi(a_t|s_t)}{\mu(a_t|s_t)}$ and $\tau_s := \max_{t,s_t} \frac{d_t^\pi(s_t)}{d_t^\mu(s_t)}$\footnote{These are really not in more precise calculations but are assumed to simplify the statement of our results.}. If $n \geq \frac{2(1-\delta)^{-1}t \tau_a \tau_s}{\max\{d_{t}^\pi(s_{t}),d_{t}^\mu(s_{t})\} }$ for all $t=2,...,H$, then for all $h=1,2,...,H$ and $s_h$, we have that:
	$$\Var[\widetilde{d}_h^\pi(s_h)]  \leq \frac{2(1-\delta)^{-1}h \tau_a\tau_s}{n} d_{h}^\pi(s_{h}).$$
\end{theorem}
\begin{proof}[Proof of Theorem~\ref{thm:fictitious_error_prop}]
	We prove by induction. The base case for $h=1$ is trivially true because
	$$
	\Var[\widetilde{d}_1^\pi(s_1)] = \frac{1}{n}(d_1^\pi(s_1) -  d_1^\pi(s_1)^2) \leq \frac{2(1-\delta)^{-1}\tau_a\tau_s}{n} d_1^\pi(s_1).
	$$
	since $\tau_a\geq 1$ and $\tau_s\geq 1$ by construction.
	
	Assume 
	$\Var[\widetilde{d}_{t}^\pi(s_{t})]  \leq \frac{2(1-\delta)^{-1}t \tau_a\tau_s}{n} d_{{t}}^\pi(s_{t})$ is true for all $t=1,...,h-1$, then by our assumption on $n$ and that $h\leq H$, we obtain that
	$$\Var[\widetilde{d}_{t}^\pi(s_{t})]  \leq d_{t}^\pi(s_{t}) \max\{d_{t}^\pi(s_{t}),d_{t}^\mu(s_{t})\} $$ for all $t=1,...,h-1$. Plug this into Corollary~\ref{cor:variance}, we get that
	\begin{align*}
	\varrho(s_t) \leq &  \sum_{s_{t-1}} \left(d_{t-1}^\pi(s_{t-1})\frac{2\max\{d_{t-1}^\pi(s_{t-1}), d_{t-1}^\mu(s_{t-1})\}}{d_{t-1}^\mu(s_{t-1})}  \sum_{a_{h-1}}\frac{\pi(a_{t-1}|s_{t-1})^2}{\mu(a_{t-1}|s_{t-1})}   \P_{t,t-1}(s_t|s_{t-1},a_{t-1})\right)\\
	\leq& 2 \tau_s\tau_a\sum_{s_{t-1}} d_{t-1}^\pi(s_{t-1}) \sum_{a_{h-1}}\pi(a_{t-1}|s_{t-1}) \P_{t,t-1}(s_t|s_{t-1},a_{t-1}) \\
	=&  2 \tau_s\tau_a d_{t}^\pi(s_{t}),
	\end{align*}
	and that 
	\begin{align*}
	\Var[\widetilde{d}_h^\pi(s_h)]  \leq& \frac{2(1-\delta)^{-1} \tau_s\tau_a}{n}\sum_{t=2}^h \sum_{s_t}  \P_{h,t}^{\pi}( s_h | s_t )^2 d_{t}^\pi(s_{t})  + \frac{1}{n}\sum_{s_1} \P_{h,1}^{\pi}( s_h | s_1)^2 d_{1}(s_1) \\
	\leq& \frac{2(1-\delta)^{-1} \tau_s\tau_a}{n} \sum_{t=1}^h  \sum_{s_t}  \P_{h,t}^{\pi}( s_h | s_t )^2 d_{t}^\pi(s_{t}) \\
	\leq& \frac{2(1-\delta)^{-1} \tau_s\tau_a}{n} \sum_{t=1}^h  \sum_{s_t}  \P_{h,t}^{\pi}( s_h | s_t ) d_{t}^\pi(s_{t}) \\
	=& \frac{2(1-\delta)^{-1} h \tau_s\tau_a}{n} d_{h}^\pi(s_{h})
	\end{align*}
	The second inequality uses that $\tau_s,\tau_a\geq 1$, the third inequality uses that $0\leq \P_{h,t}^{\pi}( s_h | s_t ) \leq 1$.	
\end{proof}

Note that the bound is tight and it implies that the error propagation is moderate. Instead of increasing exponentially, the error increases only linearly in time horizon, as long as $n$ is at least linear in $h$.

\begin{proof}[Proof of Lemma~\ref{lem:ficticious_cov}]
	We start by applying the law of total variance to obtain the following recursive equation
	\begin{align}
	\Cov[\widetilde{d}^\pi_h] &=  \E\left[ \Cov\left[  \widetilde{\P}_{h,h-1}^\pi \widetilde{d}^\pi_{h-1}   \middle|  \text{Data}_{h-1}\right]  \right]   + \Cov\left[ \E\left[\widetilde{\P}_{h,h-1}^\pi \widetilde{d}^\pi_{h-1} \middle|  \text{Data}_{h-1} \right] \right] \\
	&=\E\left[ \Cov\left[  \sum_{s_{h-1}}\widetilde{\P}_{h,h-1}^\pi(\cdot |s_{h-1}) \widetilde{d}^\pi_{h-1}(s_{h-1})   \middle|  \text{Data}_{h-1}\right]  \right]   + \Cov\left[ \E\left[\widetilde{\P}_{h,h-1}^\pi \widetilde{d}^\pi_{h-1} \middle|  \text{Data}_{h-1} \right] \right] \\
	&=  \underbrace{\E\left[ \sum_{s_{h-1}}\Cov\left[  \widetilde{\P}_{h,h-1}^\pi(\cdot | s_{h-1})\middle|  \text{Data}_{h-1}\right] \widetilde{d}^\pi_{h-1}(s_{h-1})^2 \right]}_{(***)}     +  \P_{h,h-1}^\pi \Cov[\widetilde{d}^\pi_{h-1}] [\P_{h,h-1}^\pi]^T. \label{eq:decomp_fict_marginal}
	\end{align}
	The decomposition of the covariance in the third line uses that $\Cov(X+Y) = \Cov(X) + \Cov(Y)$ when $X$ and $Y$ are statistically independent. Note that  
	$n_{s_{h-1}}$, $\widetilde{d}^\pi_{h-1}(s_{h-1}) $ are fixed and the columns of $\widetilde{\P}_{h,h-1}$ are independent when conditioning on $\text{Data}_{h-1}$.
	\begin{align}
	(***)=& \E\left[\sum_{s_{h-1}}   \Cov\left[  \frac{1}{n_{s_{h-1}}}\sum_{i=1}^n \frac{\pi(a_{h-1}^{(i)} | s_{h-1}^{(i)})}{\mu(a_{h-1}^{(i)} | s_{h-1}^{(i)})} \mathbf{1}(s_{h-1}^{(i)}=s_{h-1}) \mathbf{e}_{s_{h}^{(i)}} \middle|  \text{Data}_{h-1} \right] \mathbf{1}(E_{h-1})  \widetilde{d}_{h-1}^\pi(s_{h-1})^2\right] \\
	=&\E\left[ \sum_{s_{h-1}}   \frac{1}{n_{s_{h-1}}} \Cov\left[\frac{\pi(a_{h-1}^{(1)} | s_{h-1})}{\mu(a_{h-1}^{(1)}| s_{h-1}) }   \mathbf{e}_{s_{h}^{(1)}}   \middle|  s_{h-1}^{(1)}=s_{h-1} \right] \mathbf{1}(E_{h-1}) \widetilde{d}_{h-1}^\pi(s_{h-1})^2\right]\\
	=&  \sum_{s_{h-1}} \Bigg\{  \E\left[\frac{1}{n_{s_{h-1}}}\mathbf{1}(E_{h-1}) \widetilde{d}_{h-1}^\pi(s_{h-1})^2\right]  \Big(\sum_{a_{h-1}} \frac{\pi(a_{h-1}| s_{h-1})^2}{\mu(a_{h-1}| s_{h-1})}  \diag[ \P_{h,h-1}(\cdot|s_{h-1},a_{h-1})]  \\
	&-  \P_{h,h-1}^\pi(\cdot|s_{h-1}) [\P_{h,h-1}^\pi(\cdot|s_{h-1})]^T \Big) \Bigg\} \\
	\prec& \sum_{s_{h-1}} \Big\{  \frac{ d^\pi_{h-1}(s_{h-1})^2 + \Var[\widetilde{d}^\pi_{h-1}(s_{h-1})]}{n d^\mu_{h-1}(s_{h-1}) (1-\delta)} \sum_{a_{h-1}} \frac{\pi(a_{h-1}| s_{h-1})^2}{\mu(a_{h-1}| s_{h-1})}  \diag[ \P_{h,h-1}(\cdot|s_{h-1},a_{h-1})] \Big\}  \label{eq:dominate_by_diagonal}
	\end{align}
	The second line uses the fact that $(s_h^{(i)},a_h^{(i)})$ are i.i.d over $i$ given $s_{h-1}^{(i)}=s_{h-1}$. The third line uses law of total variance over $a_{h-1}^{(1)}$ as follows
\begin{align*}
&\Cov\left[\frac{\pi(a_{h-1}^{(1)} | s_{h-1})}{\mu(a_{h-1}^{(1)}| s_{h-1}) }   \mathbf{e}_{s_{h}^{(1)}}   \middle|  s_{h-1}^{(1)}=s_{h-1} \right]\\
=&\E\left[\left(\frac{\pi(a_{h-1}^{(1)} | s_{h-1})}{\mu(a_{h-1}^{(1)}| s_{h-1}) }\right)^2   \Cov\left[\mathbf{e}_{s_{h}^{(1)}}\middle| a_{h-1}^{(1)},s_{h-1}^{(1)}=s_{h-1}\right]   \middle|  s_{h-1}^{(1)}=s_{h-1} \right]\\
& + \Cov\left[  \frac{\pi(a_{h-1}^{(1)} | s_{h-1})}{\mu(a_{h-1}^{(1)}| s_{h-1}) }  \E\left[ \mathbf{e}_{s_{h}^{(1)}} \middle|  a_{h-1}^{(1)},s_{h-1}^{(1)}=s_{h-1} \right] \middle|  s_{h-1}^{(1)}=s_{h-1}  \right] \qquad\qquad\qquad\qquad\qquad\qquad\qquad
\end{align*}
\begin{align*}
=&\sum_{a_{h-1}}  \frac{\pi(a_{h-1} | s_{h-1})^2}{\mu(a_{h-1} | s_{h-1}) } \left[\diag(\P_{h,h-1}(\cdot| s_{h-1},a_{h-1})) -  \P_{h,h-1}(\cdot| s_{h-1},a_{h-1})\P(\cdot| s_{h-1},a_{h-1})^T\right] \\
&+ \sum_{a_{h-1}}  \frac{\pi(a_{h-1} | s_{h-1})^2}{\mu(a_{h-1} | s_{h-1}) }  \P_{h,h-1}(\cdot| s_{h-1},a_{h-1})\P_{h,h-1}(\cdot| s_{h-1},a_{h-1})^T -  \P_{h,h-1}^\pi(\cdot|s_{h-1}) [\P_{h,h-1}^\pi(\cdot|s_{h-1})]^T\\
=&\sum_{a_{h-1}}  \frac{\pi(a_{h-1} | s_{h-1})^2}{\mu(a_{h-1} | s_{h-1}) } \diag(\P_{h,h-1}(\cdot| s_{h-1},a_{h-1}))  -  \P_{h,h-1}^\pi(\cdot|s_{h-1}) [\P_{h,h-1}^\pi(\cdot|s_{h-1})]^T
\end{align*}

	
	The last line \eqref{eq:dominate_by_diagonal} follows from the fact that $\P_{h,h-1}^\pi(\cdot|s_{h-1}) [\P_{h,h-1}^\pi(\cdot|s_{h-1})]^T$ is positive semidefinite and that $\E[X^2] = \Var[X] + (\E[X])^2$. Combining \eqref{eq:decomp_fict_marginal} and \eqref{eq:dominate_by_diagonal} and by recursively apply them, we get the stated results.
\end{proof}

Combine Lemma~\ref{lem:fictitious_approximation}, \eqref{eq:reduction_to_variance} and Theorem~\ref{thm:fictitious_error_prop} with an appropriately chosen $\delta$, we get our final result:
\begin{reptheorem}{thm:main}[Main Theorem, restated]
	Let the immediate expected reward, its variance and the value function be defined as follows (for all $h=1,2,3,...,H$):
	\begin{align*}
	&r_h(s_h,a_h,s_{h+1}) := \E_\pi\left[r_h^{(1)}\middle|s_h^{(1)}=s_h, a_h^{(1)}=a_h,s_{h+1}^{(1)}=s_{h+1}\right] \in [0,R_{\max}]\\
	&\sigma_h(s_h,a_h,s_{h+1}) := \Var_\pi\left[r_h^{(1)}\middle|s_h^{(1)}=s_h, a_h^{(1)}=a_h,s_{h+1}^{(1)}=s_{h+1}\right]^{1/2}\leq \sigma\\
	&V_h^\pi(s_h) :=   \E_\pi\left[ \sum_{t=h}^{H}r_t(s_t^{(1)},a_t^{(1)}) \middle| s_{h}^{(1)}=s_h\right] \in [0,V_{\max}].
	\end{align*}
For the simplicity of the statement, define boundary conditions:
	$r_0(s_0) \equiv 0$, $\sigma_0(s_0,a_0)\equiv 0$,$\frac{d_0^\pi(s_0)}{d_0^\mu(s_0)}\equiv 1$, $\frac{\pi(a_0|s_0)}{\mu(a_0|s_0)} \equiv 1$ and $V_{H+1}^\pi \equiv 0$.
Moreover, let 	$\tau_a := \max_{t,s_t,a_t}\frac{\pi(a_t|s_t)}{\mu(a_t|s_t)}$ and $\tau_s := \max_{t,s_t} \frac{d_t^\pi(s_t)}{d_t^\mu(s_t)}$. 
	If the number of episodes $n$ obeys that
	$$n > \max\left\{ \frac{4t \tau_a \tau_s}{\min_{t,s_t}\max\{d_{t}^\pi(s_{t}),d_{t}^\mu(s_{t})\} }, \frac{16\log n}{\min_{t,s_t}d_t^\mu(s_t)}  \right\}$$ 
	for all $t=2,...,H$, then the our estimator $\widehat{v}_{\mathrm{MIS}}^\pi$ with an additional clipping step obeys that
	\begin{align*}
	\E[ (\cP\widehat{v}_{\mathrm{MIS}}^\pi -  v^\pi)^2] \leq&  \frac{1}{n}\sum_{h=0}^H \sum_{s_h}  \frac{ d_{h}^\pi(s_h)^2}{d_{h}^\mu(s_h)} \Var\left[\frac{\pi(a_h^{(1)}|s_h)}{\mu(a_h^{(1)}|s_h)} (V_{h+1}^\pi(s_{h+1}^{(1)}) +  r_h^{(1)})\middle| s_{h}^{(1)}=s_h\right] \\
	&\cdot \left(1+\sqrt{\frac{16\log n}{n\min_{t,s_t}d_t^\mu(s_t)}}\right) + \frac{19\tau_a^2\tau_s^2 S H^2(\sigma^2 + R_{\max}^2 + V_{\max}^2)}{n^2}. 
	\end{align*}
\end{reptheorem}
\begin{proof}[Proof of Theorem~\ref{thm:main}]
	Choose $\delta= \sqrt{4\log(n)/(n\min_{t,s_t}d_t^\mu(s_t))}$.
	Lemma~\ref{lem:unbiasedness_fictitious}, Lemma~\ref{lem:var_decomp_fictitious} and Theorem~\ref{thm:fictitious_error_prop} provide an MSE bound of the fictitious estimator and then by substituting the resulting bound to Lemma~\ref{lem:fictitious_approximation}, we obtain:
		%
		\begin{align}
	&\E[ (\cP \widehat{v}_{\mathrm{MIS}}^\pi  - v^\pi)^2] \\
	\leq& \frac{\Var[V_{1}^\pi(s_1^{(1)})]}{n} + \frac{(1-\delta)^{-1}}{n} \sum_{h=1}^H \sum_{s_h}  \frac{d_h^\pi(s_h)^2}{d_h^\mu(s_h)}  \Var\left[\frac{\pi(a_h^{(1)}|s_h)}{\mu(a_h^{(1)}|s_h)} (V_{h+1}^\pi(s_{h+1}^{(1)}) +  r_h^{(1)})\middle| s_{h}^{(1)}=s_h\right]  \\
	&+\frac{(1-\delta)^{-1}}{n}\sum_{h=1}^H \sum_{s_h}   \frac{2(1-\delta)^{-1}h\tau_a\tau_s}{n} \frac{d_{h}^\pi(s_h)}{d_{h}^\mu(s_h)}  \Var\left[\frac{\pi(a_h^{(1)}|s_h)}{\mu(a_h^{(1)}|s_h)} (V_{h+1}^\pi(s_{h+1}^{(1)}) +  r_h^{(1)})\middle| s_{h}^{(1)}=s_h\right]\label{eq:thm_main_proof_second_term}\\
	&+\frac{3}{n^2}HSV_{\max}^2.
	\end{align}
The first assumption on $n$ ensures that $\delta<1/2$, which allows us to write $(1-\delta)^{-1}\leq (1+2\delta)$ in the leading term and $(1-\delta)^{-1}\leq 2$ in the subsequent terms. 
  The second assumption on $n$ ensures that we can apply Theorem~\ref{thm:fictitious_error_prop} with parameter $\delta<1/2$.
	
	Then to obtain the simplified expression as stated in the theorem, we simply bound $d_h^\pi(s_h)/d_h^\mu(s_h)\leq \tau_s$ in \eqref{eq:thm_main_proof_second_term}, and then use the following bound
	\begin{align*}
	&\Var\left[\frac{\pi(a_h^{(1)}|s_h)}{\mu(a_h^{(1)}|s_h)} (V_{h+1}^\pi(s_{h+1}^{(1)}) +  r_h^{(1)})\middle| s_{h}^{(1)}=s_h\right]\\
	=&\E\Var\left[\frac{\pi(a_h^{(1)}|s_h)}{\mu(a_h^{(1)}|s_h)} (V_{h+1}^\pi(s_{h+1}^{(1)}) +  r_h^{(1)})\middle| s_{h}^{(1)}=s_h,a_{h}^{(1)},s_{h+1}^{(1)}\right]\\
	&+\Var\left[\frac{\pi(a_h^{(1)}|s_h)}{\mu(a_h^{(1)}|s_h)} (V_{h+1}^\pi(s_{h+1}^{(1)}) +  r_h(s_h,a_{h+1}^{(1)},s_{h+1}^{(1)}))\middle| s_{h}^{(1)}=s_h\right]\\
	\leq& \E_\pi\left[\frac{\pi(a_h^{(1)}|s_h)^2}{\mu(a_h^{(1)}|s_h)^2} \middle| s_h^{(1)}=s_h\right]\sigma^2 + \Var_\mu\left[\frac{\pi(a_h^{(1)}|s_h)}{\mu(a_h^{(1)}|s_h)} (V_{h+1}^\pi(s_{h+1}^{(1)}) +  r_h(s_h,a_{h+1}^{(1)},s_{h+1}^{(1)}))\middle| s_{h}^{(1)}=s_h\right]\\
	 \leq& \E_\pi\left[\frac{\pi(a_h^{(1)}|s_h)}{\mu(a_h^{(1)}|s_h)} \middle| s_h^{(1)}=s_h\right]\sigma^2 
	 + \E_\pi\left[ \frac{\pi(a_h^{(1)}|s_h)}{\mu(a_h^{(1)}|s_h)}  (V_{h+1}^\pi(s_{h+1}^{(1)}) +  r_h(s_h,a_{h+1}^{(1)},s_{h+1}^{(1)}))^2 \middle| s_h^{(1)}=s_h \right]\\
	\leq& \tau_a (\sigma^2 + 2V_{\max }^2+2R_{\max }^2).
	\end{align*}
	The second line uses the law of total expectation, the third line replaces the variance with an upper bound $\sigma^2$, the fourth line uses $\Var[X] \leq \E[X^2]$ and a change of measure from $\mu$ to $\pi$. The last line takes the upper bound $\tau_a$, $R_{\max}$ and $V_{\max}$.
	
	The proof is complete by combining the bounds of the second and the third term.
	\end{proof}
\begin{proof}[Proof of Corollary~\ref{cor:simple_bound}]	
	The results in Corollary~\ref{cor:simple_bound} requires a slightly different bound of \eqref{eq:thm_main_proof_second_term} then the one we derived above.
	We use the assumption on $n$ to ensure that
	$$
	\frac{4h\tau_a\tau_s}{n}  \frac{d_{h}^\pi(s_h)}{d_{h}^\mu(s_h)}  \leq    \frac{d_{h}^\pi(s_h) \max\{  d_{h}^\pi(s_h), d_{h}^\mu(s_h)\}}{d_{h}^\mu(s_h)} \leq \frac{d_{h}^\pi(s_h)^2}{d_{h}^\mu(s_h)}  +  d_{h}^\pi(s_h),
	$$
	which gives us an upper bound of proportional to $n^{-1} H ( \tau_a\tau_s+ \tau_a) (\sigma^2 + H^2\R_{\max}^2)$.	
\end{proof}
\begin{remark}[Sample complexity in the finite action case]
	The result implies a sample complexity upper bound (in terms of the number of episodes) of  $H^3SA/\epsilon^2$ for evaluating a fixed target policy by running an exploration policy that visits every state and action pair with probability $\Omega(1/(SA))$.  
	
	The Cramer-Rao lower bound for the discrete DAG-MDP model \citet[Theorem~3]{jiang2016doubly} implies a lower bound of $H^2SA/\epsilon^2$, which suggests that our bound is optimal up to a factor of $H$ even for the cases where $A$ is small. In the settings where $A$ is unbounded. Based on our insight with the contextual bandits setting\citep{wang2017optimal},  we conjecture that the additional dependence on $H$ in our $H^3\tau_a\tau_s/\epsilon^2$ bound is required.
	
	The comparison with the CR lower bound is a lot more delicate and interesting. We defer more detailed discussion on that to Remark~\ref{rmk:cr_lowerbound}.

\end{remark}


\begin{remark}[When $\pi\approx \mu$]\label{rmk:pi_eq_mu_case}
	It is not entirely straightforward to see how Theorem~\ref{thm:main} gives a $H^2/n$  bound in the case of $\pi\approx \mu$ rather than the $H^3/n$ bound that we describe in Corollary~\ref{cor:simple_bound}. We make it explicit here in this remark. First the variance term in the bound can be expanded using $\Var[X] = \E[X^2] - \E[X]^2$.
	\begin{align*}
	&\sum_{s_h} \frac{d^\pi(s_h)^2}{d^\mu(s_h)}\Var\left[\frac{\pi(a_h^{(1)}|s_h)}{\mu(a_h^{(1)}|s_h)} (V_{h+1}^\pi(s_{h+1}^{(1)}) +  r_h^{(1)})\middle| s_{h}^{(1)}=s_h\right]\\
	=&\sum_{s_h} \frac{d^\pi(s_h)^2}{d^\mu(s_h)}\sum_{a_h} \frac{\pi(a_h|s_h)^2}{\mu(a_h|s_h)} \Bigg(\E[V_{h+1}^\pi(s_{h+1})^2 + r_h(s_h,a_h,s_h')^2 + \sigma^2(s_h,a_h,s_h')|s_h,a_h]\\
	&\quad\quad\quad+2\E[V_{h+1}^\pi(s_{h+1})r_h(s_h,a_h,s_h')|s_h,a_h] \Bigg) -\sum_{s_h} \frac{d^\pi(s_h)^2}{d^\mu(s_h)} V_h^\pi(s_h)^2\\
	=&\sum_{s_h,a_h,s_{h+1}}\frac{d^\pi(s_h,a_h,s_{h+1})^2}{d^\mu(s_h,a_h,s_{h+1})}  \Bigg(V_{h+1}^\pi(s_{h+1})^2 + [r_h^2 + \sigma_h^2 + 2 r_h V_{h+1}^\pi](s_h,a_h,s_{h+1})\Bigg)  -  \sum_{s_h} \frac{d^\pi(s_h)^2}{d^\mu(s_h)} V_{h}^\pi(s_{h})^2.
	\end{align*}
	If we substitute the above bound into Theorem~\ref{thm:main}, we can see that the negative part of the bound getting combined with  $\sum_{s_{h-1},a_{h-1},s_h}\frac{d^\pi(s_{h-1},a_{h-1},s_{h})^2}{d^\mu(s_h,a_h,s_{h+1})}   V_{h}^\pi(s_{h})^2 $
	from the previous time point, which gives the following more interpretable upper bound of the leading term below
\begin{align*}
&	\frac{1}{n}\sum_{h=0}^H \Bigg[ \sum_{s_{h+1}} \left(\sum_{s_h,a_h}\frac{d^\pi(s_h,a_h,s_{h+1})^2}{d^\mu(s_h,a_h,s_{h+1})} -  \frac{d^\pi(s_{h+1})^2}{d^\mu(s_{h+1})}\right)V_{h+1}^\pi(s_{h+1})^2 \\
	&+\sum_{s_h,a_h,s_{h+1}}\frac{d^\pi(s_h,a_h,s_{h+1})^2}{d^\mu(s_h,a_h,s_{h+1})}  \bigg( [r_h^2 + \sigma_h^2 + 2 r_h V_{h+1}^\pi](s_h,a_h,s_{h+1})\bigg) \Bigg].
	\end{align*}
	When $\pi=\mu$, the first term goes away and the above can be bounded by 
	$$\frac{1}{n}\sum_{h=0}^H \sum_{s_h,a_h,s_{h+1}} d^\pi(s_h,a_h,s_{h+1}) (R_{\max} r_h + \sigma^2 + 2 V_1^\pi r_h) \leq \frac{1}{n}( R_{\max}V^\pi_1 + H\sigma^2 + 2[V^\pi_1]^2 ) \leq  \frac{3V_{\max}^2+H\sigma^2}{n}.$$
	Check that when $\pi$ and $\mu$ are  sufficiently close such that  $\sum_{s_{h+1}} \left(\sum_{s_h,a_h}\frac{d^\pi(s_h,a_h,s_{h+1})^2}{d^\mu(s_h,a_h,s_{h+1})} -  \frac{d^\pi(s_{h+1})^2}{d^\mu(s_{h+1})}\right) = O(1/H)$, then we get the same $H^2/n$ rate as above.
\end{remark}

\begin{remark}[Comparison to the Cramer-Rao lower bound]\label{rmk:cr_lowerbound}
	Theorem~3 in \citep[Appendix C.]{jiang2016doubly} provides a Cramer-Rao lower bound on the variance of any unbiased estimator for a simplified setting of an nonstationary episodic MDP where a reward only appear at the end of the episode and the reward is deterministic (i.e.,$\sigma^2=0$). Their bound, in our notation, translates into
	\begin{align}
	\lim_{n\rightarrow \infty}\Var\big[\sqrt{n} (\widehat{v}^\pi-v^\pi)\big] \geq \sum_{t=0}^{H}\E_\mu\left[ \frac{d^\pi(s_t^{(1)})^2}{d^\mu(s_t^{(1)})^2}\frac{\pi(a_t^{(1)}|s_t^{(1)})^2}{\mu(a_t^{(1)}|s_t^{(1)})^2} \Var\Big[V_{t+1}^\pi(s_{t+1}^{(1)})\Big| s_{t}^{(1)}, a_t^{(1)}\Big]\right].
	\end{align}
	Our Theorem~\ref{thm:main} implies
	\begin{align}
	\lim_{n\rightarrow \infty}n \mathbb{E}[ (\cP\widehat{v}_{\mathrm{MIS}}^\pi-v^\pi)^2]  \leq  \sum_{t=0}^{H}\E_\mu\left[ \frac{d^\pi(s_t^{(1)})^2}{d^\mu(s_t^{(1)})^2} \Var_\mu\Big[ \frac{\pi(a_t^{(1)}|s_t^{(1)})}{\mu(a_t^{(1)}|s_t^{(1)})}V_{t+1}^\pi(s_{t+1}^{(1)})\Big| s_{t}^{(1)}\Big]\right].
	\end{align}
	The upper and lower bounds are clearly very similar, with the only difference in where the importance weights of the actions are. We can verify that the upper bound is bigger because 
	\begin{align}
	&\Var_\mu\Big[ \frac{\pi(a_t^{(1)}|s_t^{(1)})}{\mu(a_t^{(1)}|s_t^{(1)})}V_{t+1}^\pi(s_{t+1}^{(1)})\Big| s_{t}^{(1)}\Big] \\
	=&   \E_\mu\left[ \Var\Big[\frac{\pi(a_t^{(1)}|s_t^{(1)})}{\mu(a_t^{(1)}|s_t^{(1)})}V_{t+1}^\pi(s_{t+1}^{(1)})\Big|   s_{t}^{(1)}, a_{t}^{(1)}\Big]  \middle| s_{t}^{(1)}\right] +  \Var_\mu\left[ \E\Big[\frac{\pi(a_t^{(1)}|s_t^{(1)})}{\mu(a_t^{(1)}|s_t^{(1)})}V_{t+1}^\pi(s_{t+1}^{(1)})\Big|   s_{t}^{(1)}, a_{t}^{(1)}\Big]  \middle| s_{t}^{(1)}\right] \\
	=&\E_\mu\left[ \frac{\pi(a_t^{(1)}|s_t^{(1)})^2}{\mu(a_t^{(1)}|s_t^{(1)})^2}\Var\Big[V_{t+1}^\pi(s_{t+1}^{(1)})\Big|   s_{t}^{(1)}, a_{t}^{(1)}\Big]  \middle| s_{t}^{(1)}\right]  + \Var_\mu\left[ \frac{\pi(a_t^{(1)}|s_t^{(1)})}{\mu(a_t^{(1)}|s_t^{(1)})} \E\Big[V_{t+1}^\pi(s_{t+1}^{(1)})\Big|   s_{t}^{(1)}, a_{t}^{(1)}\Big]  \middle| s_{t}^{(1)}\right]\\
	=&\E_\mu\left[ \frac{\pi(a_t^{(1)}|s_t^{(1)})^2}{\mu(a_t^{(1)}|s_t^{(1)})^2}\Var\Big[V_{t+1}^\pi(s_{t+1}^{(1)})\Big|   s_{t}^{(1)}, a_{t}^{(1)}\Big]  \middle| s_{t}^{(1)}\right]  + \Var_\mu\left[ \frac{\pi(a_t^{(1)}|s_t^{(1)})}{\mu(a_t^{(1)}|s_t^{(1)})} Q_t^\pi(s_{t}^{(1)}, a_{t}^{(1)})  \middle| s_{t}^{(1)}\right].
	\end{align}
	Provided that the second term is comparable to the first, then our upper bound is rate-optimal. 
	Both terms can be bounded by  $H^2R_{\max}^2\E_{\mu}[ \frac{\pi(a_t^{(1)}|s_t^{(1)})^2}{\mu(a_t^{(1)}|s_t^{(1)})^2}]$ and the bound cannot be improved.
	However, if we consider the overall bounds that sum over the $H$ items, the summation of the first term (the lower bound) is at most $H^2 \tau_{a}\tau_s R_{\max}^2$ (note that, somewhat surprisingly, no additional factors of $H$ is incurred), while the second term can be as large as $H^3 R_{\max}^2 \E_{\mu}[ \frac{\pi(a_t^{(1)}|s_t^{(1)})^2}{\mu(a_t^{(1)}|s_t^{(1)})^2}]$ in some cases. One trivial example of that would be an MDP that gives a constant immediate reward of $R_{\max}/2$ for all $t =  H/2+1, H/2+2,...,H$. Note that in this case, $Q_t^\pi(s_{t}^{(1)}, a_{t}^{(1)}) \equiv (H-t)R_{\max }/2$, which ensures that the second term is lower bounded by  
	$$\frac{1}{16}H^2R_{\max}^2 \Var_{\mu}\left[ \frac{\pi(a_t^{(1)}|s_t^{(1)})}{\mu(a_t^{(1)}|s_t^{(1)})} \middle|  s_t^{(1)}\right] = \frac{1}{16}H^2R_{\max}^2 \left(\E_{\mu}\left[ \frac{\pi(a_t^{(1)}|s_t^{(1)})^2}{\mu(a_t^{(1)}|s_t^{(1)})^2} \middle|  s_t^{(1)}\right] -1\right)$$
	  for all $t, s_{t}^{(1)}$. As we sum over $t$, this leads to an $H^3$ term in our upper bound that does not exist in the Cramer-Rao lower bound.
	
	A curious theoretical question is whether such an additional factor of $H$ in the error bound is required for off-policy evaluation in the small $\cS$, large $\cA$ setting that we considered.
	
	
\end{remark}



\section{Application to Other IS-Based Estimators}
\label{app:ope_framework}

In this section, we discuss the applications of our marginalized approach to other IS-based estimators.
We first unify some popular IS-based estimators, such as importance sampling and weighted doubly robust estimators, using a generic framework of IS-based estimators. Then we show the corresponding marginalized IS-based estimators, and provide the asymptotic unbiasedness and consistency results. At last, we provide details about how to deal with partial observability when applying our marginalized approach.


\subsection{Generic IS-Based Estimators Setup}
\label{sec:is_framework}

The IS-based estimators usually provide an unbiased or consistent estimate of the value of target policy $\pi$ \citep{thomas2015safe}. 
We first provide a generic framework of IS-based estimators, and analyze the similarity and difference between different IS-based estimators. This framework could give us insight into the design of IS-based estimators, and is useful to understand the limitation of them.

Let $\rho_{t}^{(i)} \coloneqq \frac{\pi(a_t^{(i)}|s_t^{(i)})}{\mu(a_t^{(i)}|s_t^{(i)})}$ be the importance ratio at time step $t$ of $i$-th trajectory, and $\rho_{0:t}^{(i)} \coloneqq \prod_{{t'} = 0}^{t}\frac{\pi(a_{t'}^{(i)}|s_{t'}^{(i)})}{\mu(a_{t'}^{(i)}|s_{t'}^{(i)})}$ be the cumulative importance ratio for the $i$-th trajectory. We also use $\rho_t(s_t,a_t)$ to denote $\pi(a_t|s_t) / \mu(a_t|s_t)$ over this paper. The generic framework of IS-based estimators can be expressed as follows
\begin{align}
\label{eq:framework}
\widehat v^{\pi} = & \frac{1}{n}\sum_{i = 1}^{n} g(s_0^{(i)}) + \sum_{i = 1}^{n} \sum_{t = 0}^{H-1} \frac{\rho_{0:t}^{(i)}}{\phi_t(\rho_{0:t}^{(1:n)})}  (r_t^{(i)} + f_t(s_t^{(i)},a_t^{(i)},s_{t+1}^{(i)})),
\end{align}
where $\phi_t: \mathbb R_+^{n} \to \mathbb R_+$ are the ``self-normalization'' functions for $\rho_{0:t}^{(i)}$, $g: \mathcal S \to \mathbb R$ and $f_t: \mathcal S \times \mathcal A \times \mathcal S \to \mathbb R$ are the ``value-related'' functions. 
Note $\mathbb{E}f_t=0$.
For the unbiased IS-based estimators, it usually has $\phi_t(\rho_{0:t}^{(1:n)}) = n$, and we first observe that the importance sampling (IS) estimator \citep{precup2000eligibility} falls in this framework using:
\begin{align}
\text{(IS)}: \qquad 
\begin{array}{l}
g(s_0^{(i)}) = 0; \, \phi_t(\rho_{0:t}^{(1:n)}) = n; \\ f_t(s_t^{(i)},a_t^{(i)},s_{t+1}^{(i)}) = 0.
\end{array}
\end{align}
For the doubly tobust (DR) estimator \citep{jiang2016doubly}, the normalization function and value-related functions are:
\begin{align}
\text{(DR)}: 
\begin{array}{l}
g(s_0^{(i)}) = \widehat V^{\pi}(s_0); \,  \phi_t(\rho_{0:t}^{(1:n)}) = n; \\
f_t(s_t^{(i)},a_t^{(i)},s_{t+1}^{(i)}) = - \widehat Q^{\pi} (s_t^{(i)},a_t^{(i)}) +  \widehat V^{\pi}(s_{t + 1}^{(i)}).
\end{array}
\end{align}

Self-normalized estimators such as weighted importance sampling (WIS) and weighted doubly robust (WDR) estimators \citep{thomas2016data} are popular consistent estimators to achieve better bias-variance trade-off. The critical difference of consistent self-normalized estimators is to use $\sum_{j = 1}^{n} \rho_{0:t}^{(j)}$ as normalization function $\phi_t$ rather than $n$. Thus, the WIS estimator is using the following normalization and value-related functions:
\begin{align}
\text{(WIS)}: \qquad
\begin{array}{l}
g(s_0^{(i)}) = 0; \, \phi_t(\rho_{0:t}^{(1:n)}) = \sum_{j = 1}^{n} \rho_{0:t}^{(j)}; \\
f_t(s_t^{(i)},a_t^{(i)},s_{t+1}^{(i)}) = 0,
\end{array}
\end{align}
and the WDR estimator:
\begin{align}
\text{(WDR)}: 
\begin{array}{l}
g(s_0^{(i)}) = \widehat V^{\pi}(s_0); \, \phi_t(\rho_{0:t}^{(1:n)}) = \sum_{j = 1}^{n} \rho_{0:t}^{(j)}; \\ f_t(s_t^{(i)},a_t^{(i)},s_{t+1}^{(i)}) = - \widehat Q^{\pi} (s_t^{(i)},a_t^{(i)}) +  \widehat V^{\pi}(s_{t + 1}^{(i)}).
\end{array}
\end{align}



Note that, the DR estimator reduced the variance from the stochasticity of action by using the technique of control variate $f_t(s_t^{(i)}, a_t^{(i)}, s_{t+1}^{(i)})$ in value-related function, and the WDR estimators reducing variance by the bias-variance trade-off using self-normalization, especially in the presence of weight clipping \citep{bottou2013counterfactual}. However, both could still suffer large variance, because the cumulative importance ratio $\rho_{0:t}^{(i)}$ always appear directly in this framework, which makes the variance to increase exponentially as the horizon goes long. 
\subsection{Marginalized IS-Based Estimators}

Recall the marginalized IS estimators \eqref{eq:objective}, we obtain a generic framework of marginalized IS-based estimators as:
\begin{align}
\label{eq:m-framework}
\widehat v_M(\pi) = & \frac{1}{n}\sum_{i = 1}^{n} g(s_0^{(i)})  + \frac{1}{n}\sum_{i = 1}^{n} \sum_{t = 0}^{H-1} \widehat w_t(s_t^{(i)}) \rho_t^{(i)}  (r_t^{(i)} + f_t(s_t^{(i)},a_t^{(i)},s_{t+1}^{(i)})).
\end{align}
Note that the ``self-normalization'' function $\phi$ has not appeared in the framework above is because we can implement the self-normalization within the estimate of $w_t(s)$. Thus, the marginalized IS-based estimators can be obtained by applying different $g$ and $f_t$ in Section \ref{sec:is_framework} into framework \eqref{eq:m-framework}.

We first show the equivalence between framework~\eqref{eq:framework} and framework~\eqref{eq:m-framework} in expectation if $\phi_t(\rho_{0:t}^{(1:n)}) = n$ and $\widehat w_t(s) = w_t(s)$.
\begin{lemma}
\label{lem:frame-equal}
If $\phi_t(\rho_{0:t}^{(1:n)}) = n$ in framework~\eqref{eq:framework} and $\widehat w_t(s) = w_t(s)$ in framework~\eqref{eq:m-framework}, then these two frameworks are equal in expectation, i.e.,
\begin{align}
& \mathbb E \left[w_t(s_t^{(i)}) \rho_t^{(i)} (r_t^{(i)} + f_t(s_t^{(i)},a_t^{(i)},s_{t+1}^{(i)}))\right] \\
= & \mathbb E \left[\rho_{0:t}^{(i)} (r_t^{(i)} + f_t(s_t^{(i)},a_t^{(i)},s_{t+1}^{(i)}))\right]
\end{align}
holds for all $i$ and $t$.
\end{lemma}

\begin{proof}[Proof of Lemma \ref{lem:frame-equal}]
Given the conditional independence in the Markov property, we have
\begin{align}
\mathbb E \left[\rho_{0:t}^{(i)} (r_t^{(i)} + f_t(s_t^{(i)},a_t^{(i)},s_{t+1}^{(i)}))\right] = & \mathbb E \left[\mathbb E \left[\rho_{0:t}^{(i)} (r_t^{(i)} + f_t(s_t^{(i)},a_t^{(i)},s_{t+1}^{(i)})) | s_{t}^{(i)} \right]\right] \\
= & \mathbb E \left[\mathbb E \left[\rho_{0:t-1}^{(i)} | s_{t}^{(i)} \right]\mathbb E \left[\rho_{t}^{(i)} (r_t^{(i)} + f_t(s_t^{(i)},a_t^{(i)},s_{t+1}^{(i)})) | s_{t}^{(i)} \right]\right] \\
= & \mathbb E \left[w_t(s_t^{(i)})\mathbb E \left[\rho_{t}^{(i)} (r_t^{(i)} + f_t(s_t^{(i)},a_t^{(i)},s_{t+1}^{(i)})) | s_{t}^{(i)} \right]\right] \\
= & \mathbb E \left[w_t(s_t^{(i)})\rho_{t}^{(i)} (r_t^{(i)} + f_t(s_t^{(i)},a_t^{(i)},s_{t+1}^{(i)}))\right],
\end{align}
where the first equation follows from the law of total expectation, the second equation follows from the conditional independence from the Markov property. This completes the proof.
\end{proof}

Next, we show that if we have an unbiased or consistent estimate $\widehat w_t$ of $w_t$, the IS-based OPE estimators that simply replace $\prod_{t' = 0}^{t - 1}\frac{\pi(a_{t'}|s_{t'})}{\mu(a_{t'}|s_{t'})}$ with $\widehat w_t(s_t)$ will remain unbiased or consistent.

\begin{theorem}
\label{thm:unbias_consis_framework}
Let $\phi_t(\rho_{0:t}^{(1:n)}) = n$ in framework~\eqref{eq:framework}, then framework~\eqref{eq:m-framework} could keep the unbiasedness and consistency same as in framework~\eqref{eq:framework} if $\widehat w_t(s)$ is an unbiased or consistent estimator for marginalized ratio $w_t(s)$ for all $t$:
\begin{enumerate}
    \item If an unbiased estimator falls in framework~\eqref{eq:framework}, then its marginalized estimator in framework~\eqref{eq:m-framework} is also an unbiased estimator of $v^{\pi}$ given unbiased estimator $\widehat w_t(s)$ for all $t$.
    \item If a consistent estimator falls in framework~\eqref{eq:framework}, then its marginalized estimator in framework~\eqref{eq:m-framework} is also a consistent estimator of $v^{\pi}$ given consistent estimator $\widehat w_t(s)$ for all $t$.
\end{enumerate}
\end{theorem}

\begin{proof}[Proof of Theorem \ref{thm:unbias_consis_framework}]
We first provide the proof of the first part of unbiasedness. Given $\mathbb E [\widehat w_t^n(s) | s] = w_t(s)$ for all $t$, then
\begin{align}
\mathbb E \left[\widehat w_t^n(s_t^{(i)}) \rho_t^{(i)}  (r_t^{(i)} + f_t(s_t^{(i)},a_t^{(i)},s_{t+1}^{(i)}))\right] = & \mathbb E \left[\mathbb E \left[\widehat w_t^n(s_t^{(i)}) \rho_t^{(i)}  (r_t^{(i)} + f_t(s_t^{(i)},a_t^{(i)},s_{t+1}^{(i)})) | s_t^{(i)} \right]\right] \\
= & \mathbb E \left[\mathbb E \left[\widehat w_t^n(s_t^{(i)}) | s_t^{(i)} \right] \mathbb E \left[ \rho_t^{(i)}  (r_t^{(i)} + f_t(s_t^{(i)},a_t^{(i)},s_{t+1}^{(i)})) | s_t^{(i)} \right]\right] \\
= & \mathbb E \left[w_t(s_t^{(i)}) \mathbb E \left[ \rho_t^{(i)}  (r_t^{(i)} + f_t(s_t^{(i)},a_t^{(i)},s_{t+1}^{(i)})) | s_t^{(i)} \right]\right] \\
= & \mathbb E \left[w_t(s_t^{(i)})\rho_{t}^{(i)} (r_t^{(i)} + f_t(s_t^{(i)},a_t^{(i)},s_{t+1}^{(i)}))\right] \\
= & \mathbb E \left[\rho_{0:t}^{(i)} (r_t^{(i)} + f_t(s_t^{(i)},a_t^{(i)},s_{t+1}^{(i)}))\right],
\label{eq:itm-unbias}
\end{align}
where the the first equation follows from the law of total expectation, the second equation follows from the conditional independence of the Markov property, the last equation follows from Lemma \ref{lem:frame-equal}. Since the original estimator falls in framework \eqref{eq:framework} is unbiased, summing \eqref{eq:itm-unbias} over $i$ and $t$ completes the proof of the first part.

We now prove the second part of consistency. Since we have
\begin{align}
& \underset {n\to \infty}{\plim} \frac{1}{n}\sum_{i = 1}^{n} \sum_{t = 1}^{H} \widehat w_t^n(s_t^{(i)}) \rho_t^{(i)}  (r_t^{(i)} + f_t(s_t^{(i)},a_t^{(i)},s_{t+1}^{(i)})) = \sum_{t = 1}^{H}  \underset {n\to \infty}{\plim} \frac{1}{n}\sum_{i = 1}^{n}  \widehat w_t^n(s_t^{(i)}) \rho_t^{(i)} (r_t^{(i)} + f_t(s_t^{(i)},a_t^{(i)},s_{t+1}^{(i)})),
\end{align}
then, to prove the consistency, it is sufficient to show
\begin{align}
\underset {n\to \infty}{\plim} \frac{1}{n}\sum_{i = 1}^{n} \widehat w_t^n(s_t^{(i)}) \rho_t^{(i)} (r_t^{(i)} + f_t(s_t^{(i)},a_t^{(i)},s_{t+1}^{(i)})) = \underset {n\to \infty}{\plim} \frac{1}{n}\sum_{i = 1}^{n} \rho_{0:t}^{(i)} (r_t^{(i)} + f_t(s_t^{(i)},a_t^{(i)},s_{t+1}^{(i)})),
\label{eq:itm-consis}
\end{align}
given $\plim_{n \to \infty} \widehat w_t^n(s) = w_t(s)$ for all $s \in \mathcal S$. Note that $d_t^{\mu}(s)$ is the state distribution under behavior policy $\mu$ at time step $t$, then for the left hand side of \eqref{eq:itm-consis}, we have
\begin{align}
& \underset{n\to \infty}{\plim} \frac{1}{n}\sum_{i = 1}^{n} \widehat w_t^n(s_t^{(i)}) \rho_t^{(i)} (r_t^{(i)} + f_t(s_t^{(i)},a_t^{(i)},s_{t+1}^{(i)})) \\
= & \sum_{s \in \mathcal S} d_t^{\mu}(s) \underset{n\to \infty}{\plim} \left[\frac{1}{n} \sum_{i = 1}^{n} \widehat w_t^n(s) \dfrac{\pi(a_t^{(i)}|s)}{\mu(a_t^{(i)}|s)} \mathbf 1 (s_t^{(i)} = s) (r_t^{(i)} + f_t(s,a_t^{(i)},s_{t+1}^{(i)})) \right] \\
= & \sum_{s \in \mathcal S} d_t^{\mu}(s) \underset{n\to \infty}{\plim} \left[\widehat w_t^n(s)\frac{1}{n} \sum_{i = 1}^{n} \dfrac{\pi(a_t^{(i)}|s)}{\mu(a_t^{(i)}|s)} \mathbf 1 (s_t^{(i)} = s) (r_t^{(i)} + f_t(s,a_t^{(i)},s_{t+1}^{(i)})) \right] \\
= & \sum_{s \in \mathcal S} d_t^{\mu}(s) \left[ \underset{n\to \infty}{\plim} \left(\widehat w_t^n(s)  \right) \underset{n\to \infty}{\plim} \left( \frac{1}{n} \sum_{i = 1}^{n} \dfrac{\pi(a_t^{(i)}|s)}{\mu(a_t^{(i)}|s)} \mathbf 1 (s_t^{(i)} = s) (r_t^{(i)} + f_t(s,a_t^{(i)},s_{t+1}^{(i)})) \right) \right] \\
= & \sum_{s \in \mathcal S} d_t^{\mu}(s) w_t(s) \underset {n\to \infty}{\plim} \left[  \frac{1}{n} \sum_{i = 1}^{n} \dfrac{\pi(a_t^{(i)}|s)}{\mu(a_t^{(i)}|s)} \mathbf 1 (s_t^{(i)} = s) (r_t^{(i)} + f_t(s,a_t^{(i)},s_{t+1}^{(i)})) \right]\\
= & \sum_{s \in \mathcal S} d_t^{\mu}(s) w_t(s) \mathbb E \left[ \dfrac{\pi(a_t|s)}{\mu(a_t|s)} (r_t + f_t(s,a_t,s_{t+1})) \Big| s_t = s \right],
\label{eq:itm-consis-left}
\end{align}
where the first equation follows from the weak law of large number. Similarly, for the right hand side of \eqref{eq:itm-consis}, we have
\begin{align}
& \underset {n\to \infty}{\plim} \frac{1}{n}\sum_{i = 1}^{n} \rho_{0:t}^{(i)} (r_t^{(i)} + f_t(s_t^{(i)},a_t^{(i)},s_{t+1}^{(i)})) \\
= & \sum_{s \in \mathcal S} d_t^{\mu}(s) \underset{n\to \infty}{\plim} \left[ \frac{1}{n}\sum_{i = 1}^{n} \prod_{t'=0}^{t-1} \dfrac{\pi(a_{t'}^{(i)}|s_{t'}^{(i)})}{\mu(a_{t'}^{(i)}|s_{t'}^{(i)})} \mathbf 1 (s_t^{(i)} = s) \dfrac{\pi(a_t^{(i)}|s)}{\mu(a_t^{(i)}|s)} (r_t^{(i)} + f_t(s,a_t^{(i)},s_{t+1}^{(i)})) \right] \\
= & \sum_{s \in \mathcal S} d_t^{\mu}(s) \mathbb E \left[ \prod_{t'=0}^{t-1} \dfrac{\pi(a_{t'}|s_{t'})}{\mu(a_{t'}|s_{t'})} \dfrac{\pi(a_t|s)}{\mu(a_t|s)} (r_t + f_t(s,a_t,s_{t+1}))\Big| s_t = s  \right] \\
= & \sum_{s \in \mathcal S} d_t^{\mu}(s) \mathbb E \left[ \prod_{t'=0}^{t-1} \dfrac{\pi(a_{t'}|s_{t'})}{\mu(a_{t'}|s_{t'})} \Big| s_t = s \right] \mathbb E \left[\dfrac{\pi(a_t|s)}{\mu(a_t|s)} (r_t + f_t(s,a_t,s_{t+1})) \Big| s_t = s \right] \\
= & \sum_{s \in \mathcal S} d_t^{\mu}(s) w_t(s) \mathbb E \left[\dfrac{\pi(a_t|s)}{\mu(a_t|s)} (r_t + f_t(s,a_t,s_{t+1})) \Big| s_t = s \right],
\label{eq:itm-consis-right}
\end{align}
where the first equation follows from the weak law of large number and the third equation follows from the conditional independence of the Markov property. Thus, we have \eqref{eq:itm-consis-left} equal to \eqref{eq:itm-consis-right}. This completes the proof of the second half.
\end{proof}

In partially observable MDPs (POMDPs), we may not be able to obverse all states. However, if there exist any observable states, our marginalized approach could leverage these observable states to reduce variance. That is, we use the partial trajectory from the closest observable states to the current time step to represent the current state. Assume the current time step is $t$ and the closest observable states is $s_{t - L}$ at time step $t - L$, then we can use $\frac{d_{t}^{\pi}(s_{t - L})}{d_{t}^{\mu}(s_{t - L})} \prod_{i = t - L}^{t - 1} \frac{\pi(a_i|s_i)}{\mu(a_i|s_i)}$ as $w_t(s_t)$, while other IS-based methods are equivalent to using $\prod_{0}^{t - 1} \frac{\pi(a_i|s_i)}{\mu(a_i|s_i)}$ as $w_t(s_t)$. The observable states in POMDPs can be considered as the states that can be reunioned at in the DAG MDPs.
If there is no observable state in POMDPs, then it is equivalent that DAG MDPs is reduced to tree MDPs.
Definition of DAG and Tree MDPs can be found in the extended version of  \citep{jiang2016doubly}.


Finally, we propose a new marginalized IS estimator to further improve the data efficiency and reduce variance. Since DR only reduces the variance from the stochasticity of action \citep{jiang2016doubly} and our marginalized estimator~\eqref{eq:m-framework} reduce the variance from the cumulative importance weights, it is also possible to reduce the variance the stochasticity of reward function.

Based on the definition of MDPs, we know that $r_t$ is the random variable that only determined by $s_t,a_t$. Thus, if $\widehat R(s,a)$ is an unbiased and consistent estimator for $R(s,a)$, $r_t^{(i)}$ in framework~\eqref{eq:m-framework} can be replaced by that $\widehat R(s_t^{(i)},a_t^{(i)})$, and keep unbiasedness or consistency same as using $r_t^{(i)}$. 

Note that we can use an unbiased and consistent Monte-Carlo based estimator
\begin{align}
\widehat{r}(s_t,a_t) = \frac{\sum_{i=1}^n r_t^{(i)}\mathbf{1}(s_t^{(i)}=s_t, a_t^{(i)} =a_t)}{\sum_{i=1}^n \mathbf{1}(s_t^{(i)}=s_t, a_t^{(i)} =a_t) },
\end{align}
and then we obtain a better marginalized framework
\begin{align}
\label{eq:bm-framework}
& \widehat v_{BM}(\pi) = \frac{1}{n}\sum_{i = 1}^{n} g(s_0^{(i)}) + \frac{1}{n}\sum_{i = 1}^{n} \sum_{t = 0}^{H-1} \widehat w_t(s_t^{(i)}) \rho_t^{(i)}  (\widehat{r}(s_t^{(i)},a_t^{(i)}) + f_t(s_t^{(i)},a_t^{(i)},s_{t+1}^{(i)})).
\end{align}

\begin{remark}
Note that, the only difference between~\eqref{eq:m-framework} and~\eqref{eq:bm-framework} is $r_t^{(i)}$ and $\widehat{r}(s_t^{(i)},a_t^{(i)})$. Thus, the unbiasedness or consistency of~\eqref{eq:bm-framework} can be obtained similarly by following Theorem~\ref{thm:unbias_consis_framework} and its proof. 
%
\end{remark}

One interesting observation is that when each $(s_t,a_t)$-pair is observed only once in $n$ iterations, then framework~\eqref{eq:bm-framework} reduces to~\eqref{eq:m-framework}. Note that when this happens, we could still potentially estimate $\widehat w_t^n(s_t)$ well if $|\mathcal A|$ is large but $|\mathcal S|$ is relative small, in which case we can still afford to observe each potential values of $s_t$ many times. Thus, we can also obtain better marginalized IS-based estimators, e.g., the MIS and MDR estimators we use in our experiments, by applying different $g$ and $f_t$ in Section \ref{sec:is_framework} into framework \eqref{eq:bm-framework}.

\section{Details of Experiments}
\label{sec:exp_app}

In this section, we first clarify the experiment settings. We also provide a detailed discussion about the preference of MIS and SSD-IS. Finally, we provide the extended experiential results about applying MIS to doubly robust related approaches.

\subsection{Environment Settings}

\paragraph{ModelWin MDP}
As depicted in Figure~\ref{fig:ModelWinMDP-main}, the agent in the ModelWin domain always begins in $s_1$, where it must select between two actions. The first action $a_1$ causes the agent to transition to $s_2$ with probability $p$ and $s_3$ with probability $1-p$. The second action $a_2$ does the opposite. We set $p=0.4$. The agent receives a reward of $1$ every time the state transitions to $s_2$, $-1$ to $s_3$, and $0$ otherwise.

\paragraph{ModelFail MDP}
The dynamics of ModelFail MDP (Figure~\ref{fig:ModelFailMDP-main}) is similar to ModelWin, but the reward is delayed after the unobservable states --- the agent receives a reward of $1$ only when it arrives $s_1$ from the left state and $-1$ only when it arrives $s_1$ from the right state. We set $p=1$ to make the problem easier. 

Policy $\pi$ takes action $a_1$ and $a_2$ with probabilities $0.2$ and $0.8$ when at state $s_1$ or observing ``$?$''. $\mu$ take actions uniformly at random. 

We remark that the partial-state observability in \textbf{ModelFail} is specialized and should be distinguished from the more general partial observability considered in the classical POMDP literature. The two distinctive (and clearly artificial) differences are that 
\begin{enumerate}
    \item There are checkpoints of full state observability every other step.
        \item We assume that the action probability is logged when the observation is ``$?$''. 
\end{enumerate}
The model-based approach clearly fails when ``$?$'' is treated as if it is a state when building the model. A standard POMDP that uses just a memory of size 2 will resolve this issue without any trouble. One may also consider an alternative MDP that only takes the checkpoint states $s_1$ as states, but the two actions that the policies will take are no longer a function of just $s_1$ (in this example it actually is because the observation is always ``$?$'' in the step after $s_1$).

Finally, both \textbf{ModelWin} and \textbf{ModelFail} are highly specialized examples with deterministic transitions into the states $s_1$ that could potentially generate rewards for some actions. Moreover, there are no non-trivial actions involved as we transition from $s_2$ and $s_3$ back to $s_1$. This means that we can perfectly estimate the marginal state-distribution of $s_1$ with just one data point in all methods. As a result, we do not expect the results to reveal the worst-case dependence on the model parameters such as $H$. The following example fixes that.

\paragraph{Non-stationary Non-mixing MDP}
In the time-varying MDP example, we consider the following carefully designed MDP where there are two states and a continuous action in $[0,1]$.
In State $0$ the agent always transitions to State $0$, regardless of the actions. In State $1$, it transitions to State $0$ deterministically if the action is taken to be within an unknown subset of measure $1/H$ within $[0,0.5]$. This subset might be different for different $t$. When the agent is at State $0$, then a reward of $1$ is received regardless of the actions taken when the step number is larger than $H/2$; otherwise no reward is received.

The behavior and target policy (probability density on $[0,1]$) are defined to be
$$
\mu(a|s=1) = 1 \text{ for all }a,
$$
and
$$
\pi(a|s=1) = \begin{cases}
1.9 &\text{ if } a\in[0,0.5]\\
0.1 &\text{ otherwise.}
\end{cases}
$$
and $\pi(a|s=0)=\mu(a|s=0) = 1$.

This example is deliberately designed such that we have a non-stationary dynamics\footnote{The transition matrices on both $\pi$ and $\mu$ are actually stationary.} that does not really mix beyond a constant factor so an additional factor of $H$ in the sample complexity can potentially appear. Meanwhile, the cumulative reward is proportional to $H$ for the target policy, so we expect to see a $H^1.5$ dependence in the (relative) RMSE curves as we vary $H$. Finally, due to the non-mixing property of this example, and the importance weight of stationary distributions based on SSD-IS is expected to be biased.
All the above observations are consistent with what we see in the experiments presented in Figure~\ref{fig:tvar}.

\paragraph{Mountain Car}
Mountain Car domain is a classic control problem with 2-dimensional state space (position and velocity), 3 discrete one-dimensional actions (push left, no push, push right), and deterministic dynamics. We follow the same dynamic as in \citep{sutton1998reinforcement}. The horizon, $H$, is set to be 100. We use initial state distribution to be uniform in position and $0$ in velocity to ensure exploratory. Since our proposed method mainly focuses on the tabular setting, we use the state aggregation for both MIS and SSD-IS to ensure fair comparison: position is multiplied by $2^6$ and velocity is multiplied by $2^8$, and then we use the rounded integers to be the abstract state (adopted from \citep{jiang2016doubly}). Thus, the (marginalized or stationary) state distribution can be estimated on the tabular abstract states.

\subsection{Detailed Discussions}

The ModelWin domain is only a very special case of episodic fully-observable MDPs. Even if we use the stationary state distribution (estimated by ignoring the within-episode step count in the dataset) instead of marginalized state distribution in \eqref{eq:MIS_value}, that value will  still happen to be correct in both time-invariant and time-varying case. However, that is not correct in general. As the results we showed in the Mountain Car domain, SSD-IS fails to provide correct evaluation. That is because SSD-IS is designed for the infinite-horizon problems and usually cannot be directly applied to the episodic problems, where SSD-IS uses the stationary distribution ($t\rightarrow \infty$) to approximate that for all $t=1,...,H$ which is biased and not consistent even as the number of episodes $n\rightarrow \infty$ in general. For example, in the Mountain Car domain, the stationary state distribution, $\prod_{t = 1}^{\infty}P_{s_{t - 1},s_t}^{\pi}d_0$, will converge to the probability mass on the absorbing state with any exploratory policy $\pi$.

The result of mountain car experiment in the current version is slightly different from the early version. There are two main modifications in that experiments: 1. In the early version, the on-policy estimated $v^\pi$ for calculating RMSE did not use enough trajectories, so that the curves in the early version are biased. 2. We changed the implementation of SSD-IS in our current version. Previously, we solved Equation (8) and (9) in \citep{liu2018breaking} using an iterative approach. The current implementation solves Equation (8) and (9) in \citep{liu2018breaking} directly by re-formalizing it to be a quadratic programming problem. The current implementation follows the released code provided by the author of \citep{liu2018breaking}, and the detailed description of that can be found in Appendix \ref{sec:ssd_is_details}.

We also explain the reason of MIS outperforming MDR in Figure \ref{fig:tinv_app} and Figure \ref{fig:tvar_MountainCar_n_app}. In the MDR methods, we split our dataset into two halves. We use one half to estimate the marginalized state distribution, and the other half to estimate the Q-function. Intuitively, since Q-function is only used to be a control variant in the estimator, we suppose the statistical error from the marginalized state distribution may dominate the overall statistical error. As MIS uses the whole data to estimate the marginalized state, whereas MDR only uses a half data, the statistical error of MIS could be lass that of MDR. The theoretical explanation of that goes beyond the topic of this paper, and we will leave it as the future work.
 
\subsection{SSD-IS with finite state space.}\label{sec:ssd_is_details}
The pioneering work by \citet{liu2018breaking} describes a method --- SSD-IS --- for estimating the ratio of stationary state distribution under $\pi$ and $\mu$ for an infinite horizon (possibly discounted) MDP.

The estimator is described primarily for the case when the state is a continuous variable, which requires defining a reproducing kernel Hilbert space (RKHS) and solving a mini-max problem. 

To be a bit more self-contained in this paper, we provide the concise formula using our notation for estimating the stationary distribution $d_{\infty}^\pi(s)$ as well as directly estimating the importance ratio
$$
\rho(s):=\frac{d_{\infty}^\pi(s)}{d_{1:N-1}^\mu(s)}
$$
between the $d_{\infty}^\pi(s)$ and the marginalized state distribution under $\mu$ that measures the average state-visitation within the first $N$ iterations (note that we can observe triplets $(s_t,a_t,s_{t+1})$ for all $t=1,...,N-1$).

Note that a roll-out in an infinite-horizon environment of a fixed length $N$ can be denoted in our notation with a single episode $n=1$ and horizon $H=N$.

The master equation that we need is the following:
\begin{equation}
    d_{\infty}^\pi(s') = P^\pi(s'|s)  d_{\infty}^\pi(s) = P^\pi(s'|s)d^\mu_{1:H-1}(s) \frac{d_{\infty}^\pi(s)}{d_{1:H-1}^\mu(s)},
\end{equation}
which, in matrix form, is:
\begin{equation}\label{eq:eigen_dpi}
    d_{\infty}^\pi  = A^{\pi,\mu} \text{Diag}(d_{1:H-1}^{\mu})^{-1} d_{\infty}^\pi 
\end{equation}
where $A^{\pi,\mu}\in\R^{S\times S}$ and $A^{\pi,\mu}(s',s)$ measures the joint distribution of $s \sim d_t^{\mu}$ with a randomly chosen $t$ from $1,...,H-1$ and $s'$ that is obtained by taking an action according to $\pi$ at $s$.

By left-multiplying  $\text{Diag}(d_{1:H-1}^{\mu})^{-1}$ on both sides of the equation, we also get
\begin{equation}\label{eq:eigen_rho}
    \rho  = \text{Diag}(d_{1:H-1}^{\mu})^{-1} A^{\pi,\mu}  \rho.
\end{equation}

Observe that \eqref{eq:eigen_dpi} and \eqref{eq:eigen_rho} are eigenvalue problems of $d_{\infty}^\pi$ and $\rho$. They differ only by whether we normalize the $A^{\pi,\mu}$ matrix on the left or on the right by multiplying the diagonal matrix $\text{Diag}(d_{1:H-1}^{\mu})^{-1}$.

They suggest that if we can consistently estimate $A^{\pi,\mu}$ and $\text{Diag}(d_{1:H-1}^{\mu})^{-1}$, then \textbf{the right eigenvector of the corresponding estimated matrices with eigenvalue closest to $1$} will be consistent estimators of $d_{\infty}^\pi$ and $\rho$.

Note that we can estimate the joint-distribution
$A^{\pi,\mu}[s',s]$ by importance sampling using
$$
\hat{A}^{\pi,\mu}[s',s] = \frac{1}{n}\sum_{i=1}^n\frac{1}{H-1}\sum_{t=1}^{H-1} \frac{\pi(a_t^{(i)}|s_t^{(i)})}{\mu(a_t^{(i)}|s_t^{(i)})} \mathbf{1}(s_{t}^{(i)}=s, s_{t+1}^{(i)}=s')
$$
with potentially infinite action.
And $d_{1:H-1}^{\mu}$ can be estimated by
$$
\hat{d}_{1:H-1}^{\mu}(s) = \frac{1}{n}\sum_{i=1}^n\frac{1}{H-1}\sum_{t=1}^{H-1}  \mathbf{1}( s_{t}^{(i)}=s).
$$
For the infinite horizon case, we can just take $n=1$.

We emphasize that while the above results and the spectral estimators were not explicitly presented by \citet{liu2018breaking}, they are simply a rewriting of Equation (8) and (9) in \citep{liu2018breaking} more explicitly in a more specialized case.

The SSD-IS implementation that we used in the experiments with discrete state space corresponds to this particular version that we described in this section, which is the same as the version of the code released by the authors modulo some boundary conditions\footnote{In the released code provided by the authors of \citep{liu2018breaking}, there is a version of SSD-IS implemented for the discrete state space that first estimates $d_{\infty}^\pi(s)$ than output the importance weights to be the ratio of this estimate and $\hat{d}_{1:H-1}^{\mu}$ (see \url{https://github.com/zt95/infinite-horizon-off-policy-estimation/blob/master/taxi/Density_Ratio_discrete.py}). However, $\hat{d}_{\infty}^\pi(s)$ is slightly different from the spectral algorithm that we described and it provides a mysterious result that is inconsistent with the stationary distribution that we derived analytically by hands in the example we considered in Figure~\ref{fig:tvar} ($d_{\infty}^\pi(s=1) = 1$ with large probability, while the estimated value by running that piece of code is far off).}. These boundary conditions seem to be important for getting SSD-IS to work correctly for the finite horizon MDPs.

That said, we acknowledge that when the underlying MDP is stationary and $H$ is large enough relative to the mixing rate of the MDP, then using the estimated importance weight $\rho$ to construct importance sampling estimators as in SSD-IS may provide a favorable bias-variance trade-off in finite sample, because its variance is smaller by a factor of $H$ than the standard MIS while its bias on the estimated importance ratio $\frac{d^\pi_t(s)}{d^\mu_t(s)}$ decays exponentially as $t$ gets larger.

\subsection{Extended Experimental Studies}

We now present further empirical results. To test the use of our approach in other IS-based estimators, we compared DR, WDR, MDR, and MIS in the same environments, where DR denotes the doubly robust estimator \citep{jiang2016doubly}, WDR denotes the weighted doubly robust estimator \citep{thomas2016data}, MIS denotes the estimator using proposed marginalized approach used with doubly robust, and MIS is our marginalized importance sampling estimator. The estimates of $d_t^{\pi}$ and $d_t^{\mu}$ are projected to the probability simplex in our MDR and MIS estimators. The results are obtained in the same environments as Section \ref{sec:exp}.

\begin{figure*}[thb]
    \centering
    \begin{subfigure}[b]{1\linewidth}
        \centering
        \includegraphics[width=0.5\linewidth]{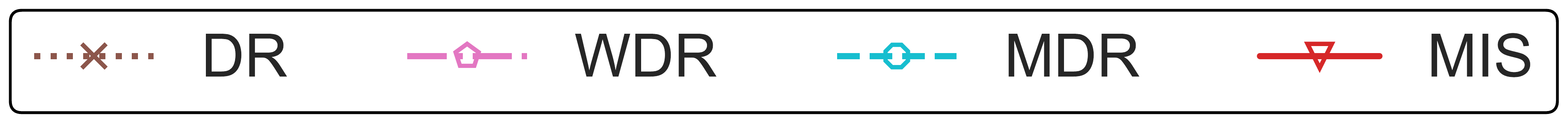}
    \end{subfigure}\\
    \begin{subfigure}[b]{0.2446\linewidth}
        \centering
        \includegraphics[width=\linewidth]{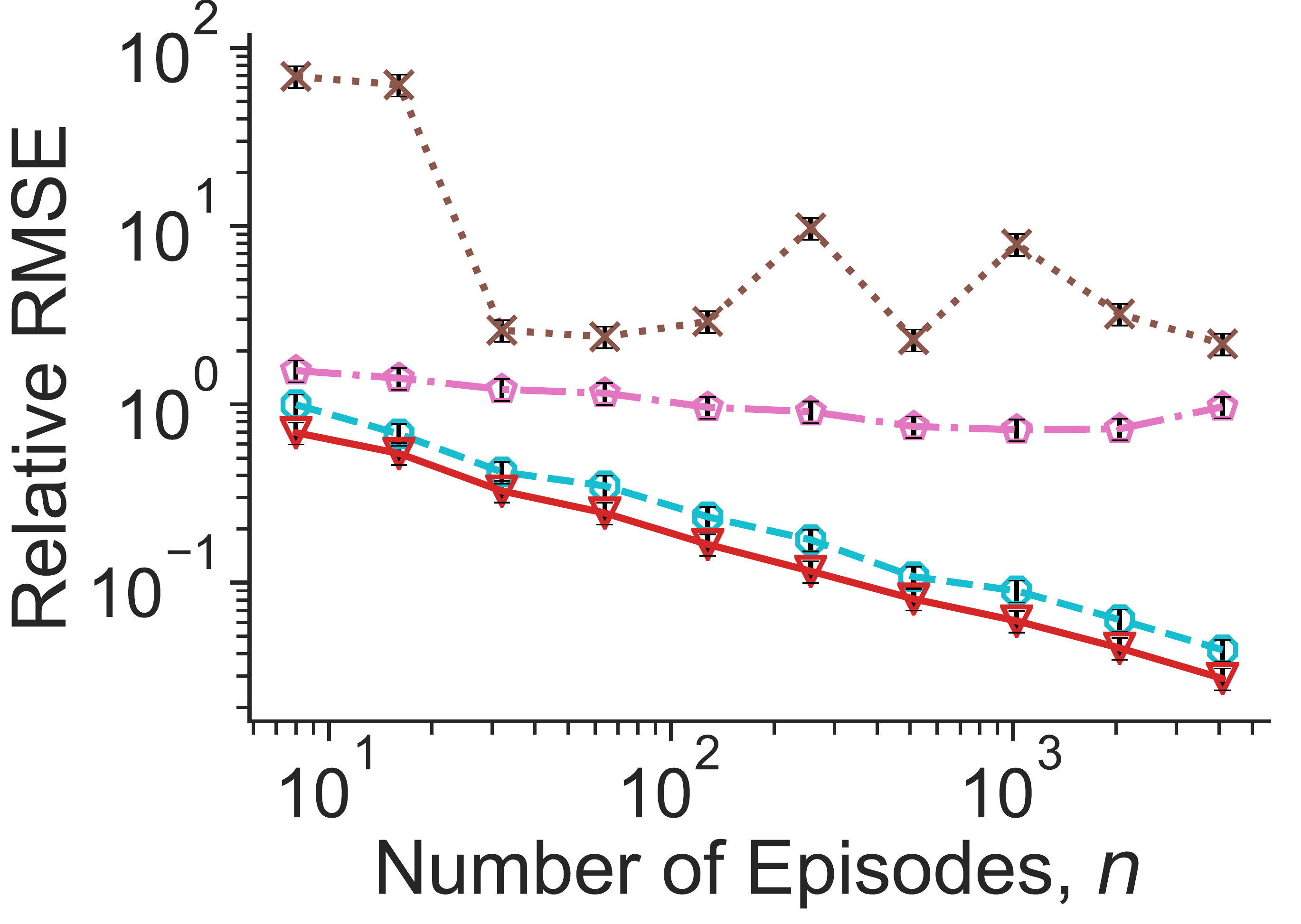}
        \caption{ModelWin MDP with different number of episodes, $n$}
        \label{fig:tinv_ModelWin_n_app}
    \end{subfigure}
    \begin{subfigure}[b]{0.2446\linewidth}
        \centering
        \includegraphics[width=\linewidth]{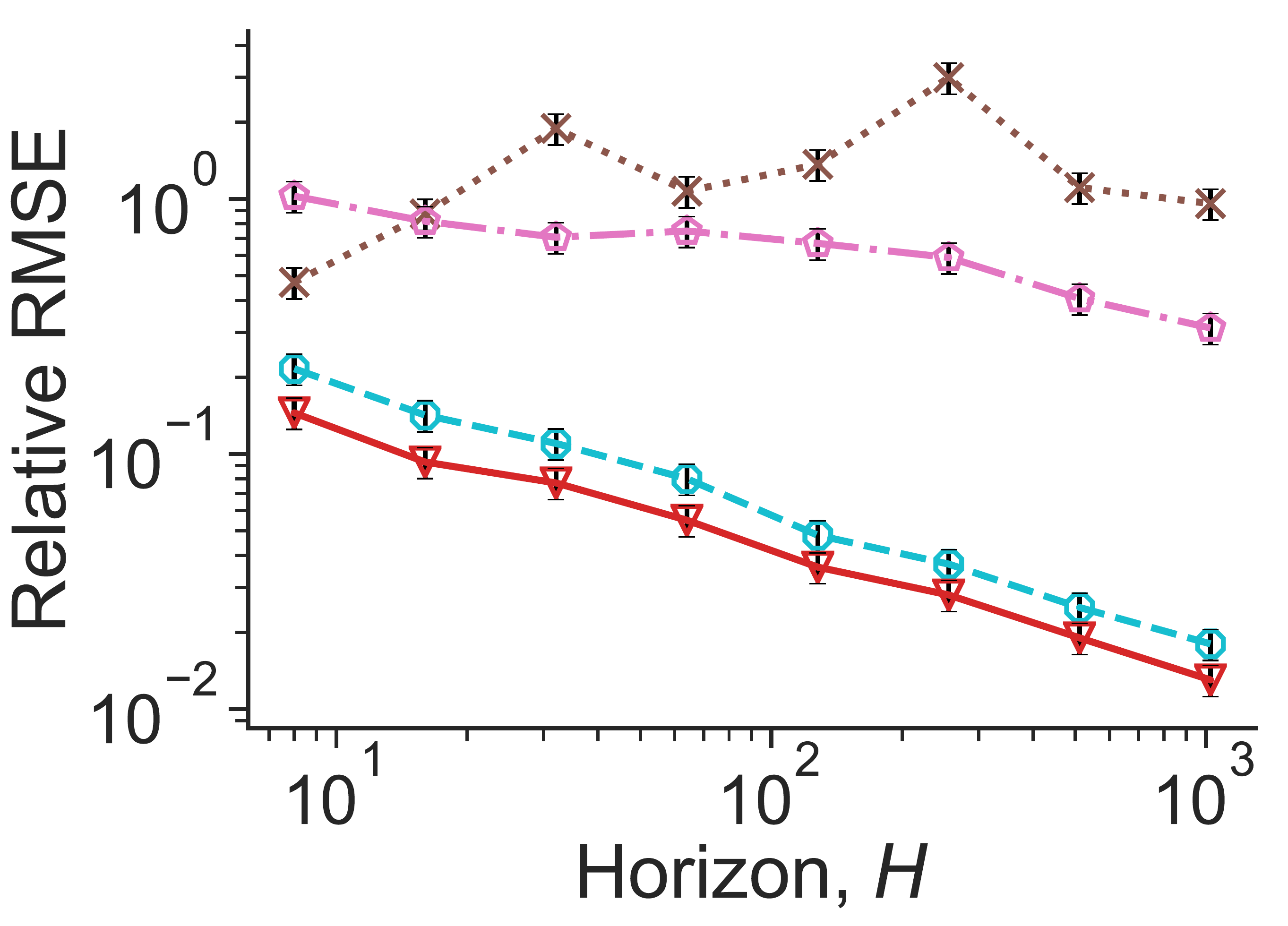}
        \caption{ModelWin MDP with different horizon, $H$}
        \label{fig:tinv_ModelWin_H_app}
    \end{subfigure}
    \begin{subfigure}[b]{0.2446\linewidth}
        \centering
        \includegraphics[width=\linewidth]{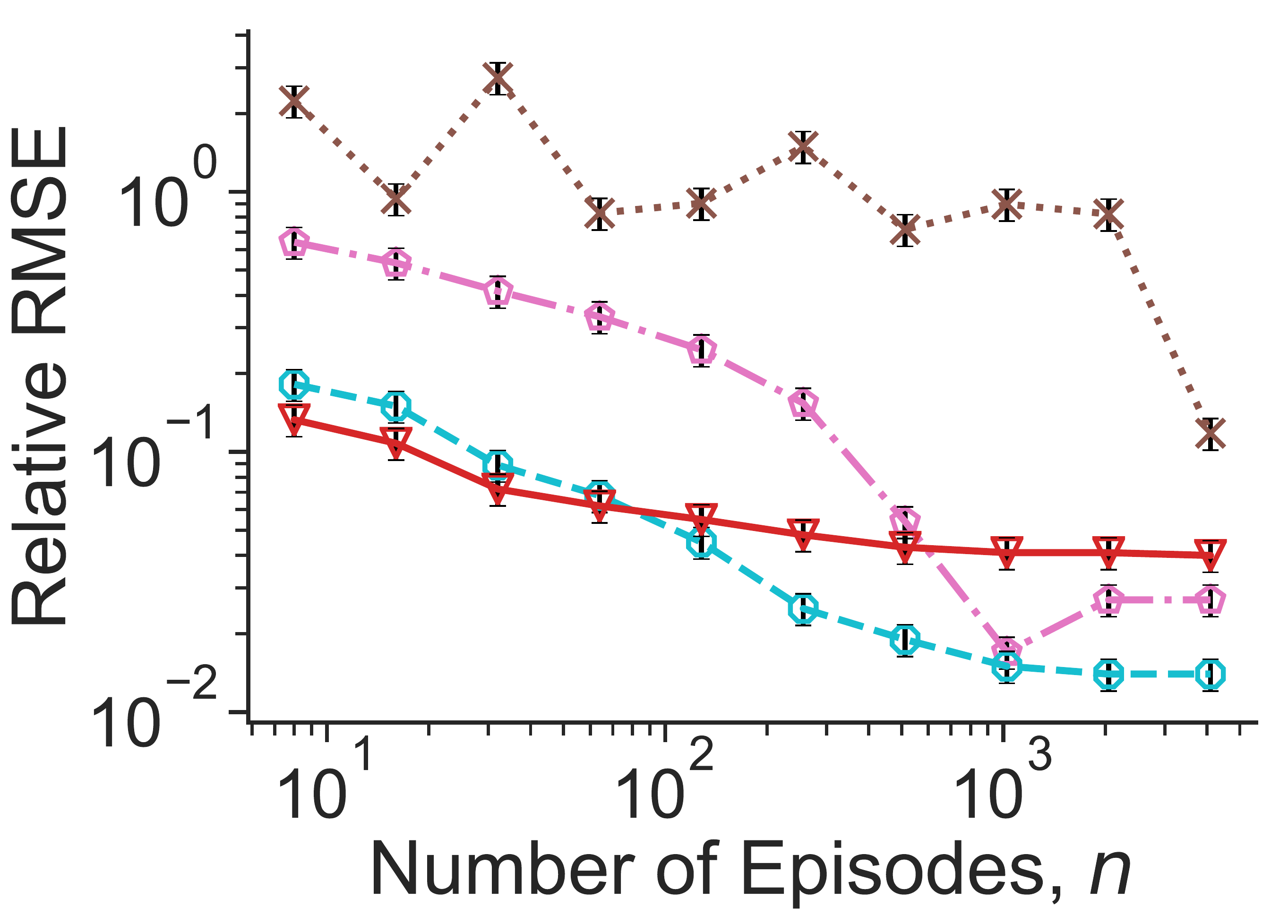}
        \caption{ModelFail MDP with different number of episodes, $n$}
        \label{fig:tinv_ModelFail_n_app}
    \end{subfigure}
    \begin{subfigure}[b]{0.2446\linewidth}
        \centering
        \includegraphics[width=\linewidth]{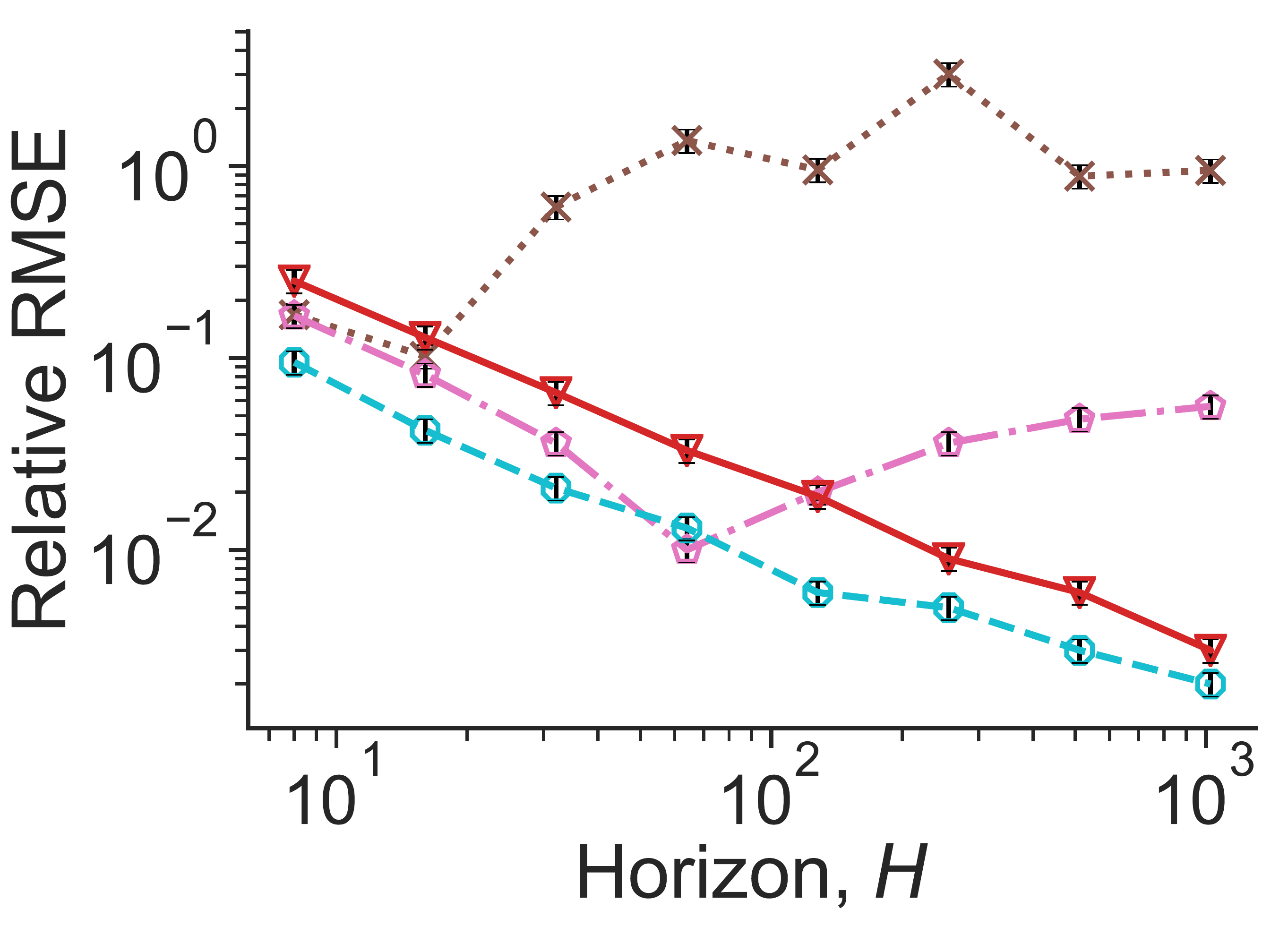}
        \caption{ModelFail MDP with different horizon, $H$}
        \label{fig:tinv_ModelFail_H_app}
    \end{subfigure}
    \caption{Results on Time-invariant MDPs.}
    \label{fig:tinv_app}
\end{figure*}


\begin{figure}[htb]
    \centering
    \includegraphics[width=0.4\linewidth]{figures/legends_app.pdf}\\
    \centering
    \includegraphics[width=0.3\linewidth]{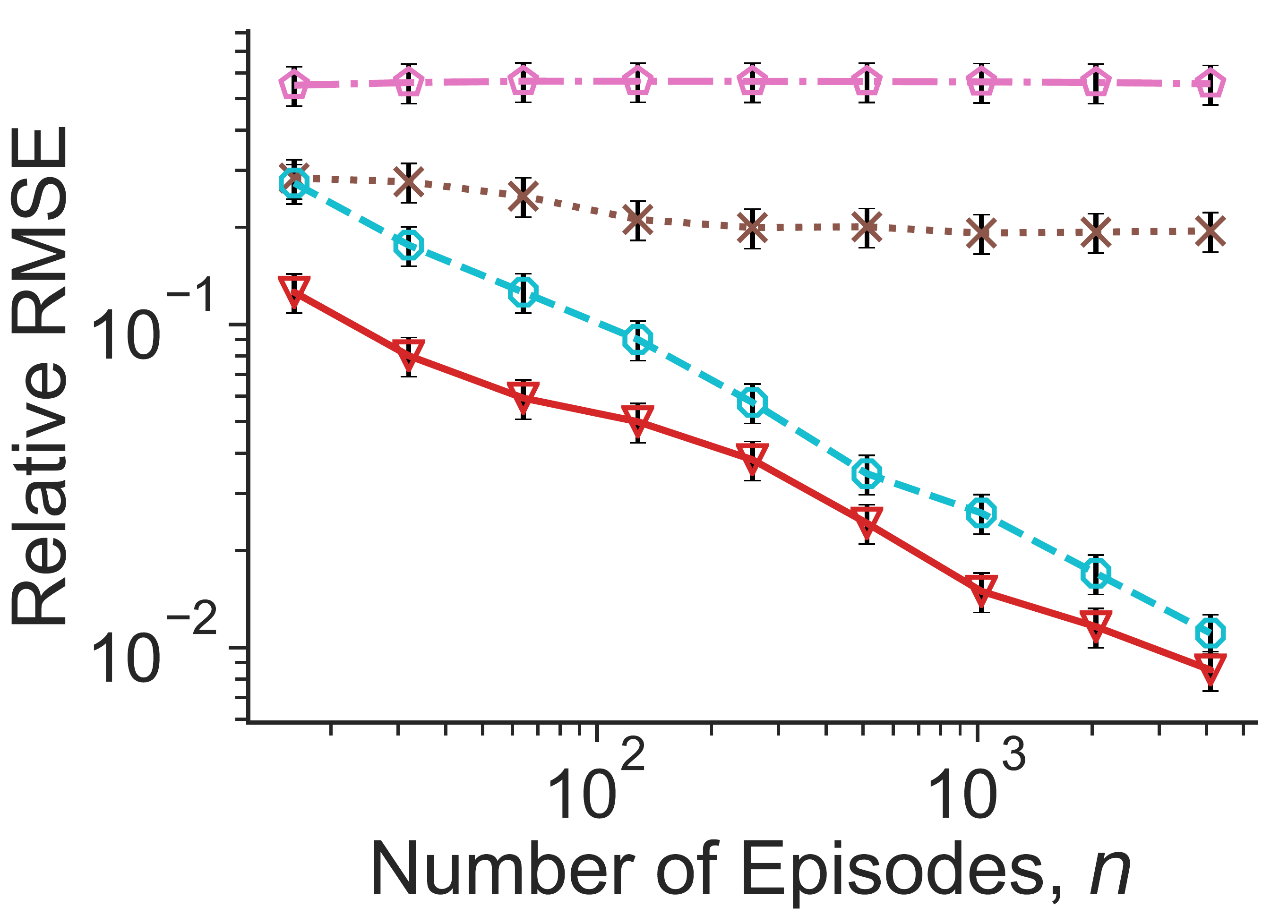}
    \caption{Mountain Car with different number of episodes.}
    \label{fig:tvar_MountainCar_n_app}
\end{figure}

The results are in Figure \ref{fig:tinv_app} and Figure \ref{fig:tvar_MountainCar_n_app}. These demonstrate that other IS based methods can also leverage our marginalized approach to benefit performance dramatically.

\section{Algorithm Details}
\label{app:alg}

\begin{algorithm*}[thb]
\caption{Marginalized Off-Policy Evaluation}
\label{alg:mainalgo}
{\bfseries Input:} Transition data $\mathcal D = \{\{s_t^{(i)},a_t^{(i)},r_t^{(i)},s_{t + 1}^{(i)}\}_{t = 0}^{H - 1}\}_{i = 1}^{n}$ from the behavior policy $\mu$. A target policy $\pi$ which we want to evaluate its cumulative reward.
\begin{algorithmic}[1]
\STATE Calculate the on-policy estimation of $d_0(\cdot)$ by
\begin{align}
\widehat d_0(s) = \frac{1}{n}\sum_{i = 1}^{n} \mathbf{1}(s_0^{(i)} = s),
\end{align}
and set $\widehat d_0^{\mu}(\cdot)$ and $\widehat d_0^{\pi}(\cdot)$ as $\widehat d_0(s)$.
\FOR{$t = 0,1,\dotsc,H-1$}
\STATE Choose all transition data as time step $t$, $\{s_t^{(i)},a_t^{(i)},r_t^{(i)},s_{t + 1}^{(i)}\}_{i = 1}^{n}$.
\STATE Calculate the on-policy estimation of $d_{t + 1}^{\mu}(\cdot)$ by
\begin{align}
\widehat d_{t + 1}^{\mu}(s) = \frac{1}{n}\sum_{i = 1}^{n} \mathbf{1}(s_{t + 1}^{(i)} = s).
\end{align}
Calculate the off-policy estimation of $d_{t + 1}^{\pi}(\cdot)$ by
\begin{align}
\widehat d_{t + 1}^{\pi}(s) = & \frac{1}{n} \sum_{i = 1}^{n}\frac{\widehat d_{t}^{\pi}(s_t^{(i)})}{\widehat d_{t}^{\mu}(s_t^{(i)})}\frac{\pi(a_t^{(i)}|s_t^{(i)})}{\mu(a_t^{(i)}|s_t^{(i)})} \mathbf 1(s_{t+1}^{(i)} = s)
\label{eq:algo_mis}
\end{align}
\STATE Estimate the reward function
\begin{align}
\widehat{r}(s_t,a_t) = \frac{\sum_{i=1}^n r_t^{i}\mathbf{1}(s_t^{i}=s_t, a_t^{i} =a_t)}{\sum_{i=1}^n \mathbf{1}(s_t^{i}=s_t, a_t^{i} =a_t) }.
\end{align}
\STATE Normalize $d_{t + 1}^{\pi}(\cdot)$ into the probability simplex, and specify $\widehat w_{t + 1}(s)$ as $\dfrac{\widehat d_{t + 1}^{\pi}(s)}{\widehat d_{t + 1}^{\mu}(s)}$ for each $s$.
\ENDFOR
\STATE Substitute the all estimated values above into \eqref{eq:bm-framework} to obtain $\widehat v(\pi)$, the estimated cumulative reward of $\pi$.
\end{algorithmic}
\end{algorithm*}

Algorithm~\ref{alg:mainalgo} summarizes our method of marginalized off-policy evaluation. Note that the MIS estimator in Section~\ref{sec:exp} is using the estimate of $d_t^{\pi}(\cdot)$ by normalizing \eqref{eq:algo_mis} into the probability simplex for better performance.

\end{document}